\documentclass{ecai} 

\usepackage{latexsym}
\usepackage{amssymb}
\usepackage{amsmath}
\usepackage{amsthm}
\usepackage{booktabs}
\usepackage{enumitem}
\usepackage{graphicx}
\usepackage{color}
\usepackage{booktabs}						
\usepackage{multirow}						
\usepackage{amsfonts}						
\usepackage{duckuments}					
\usepackage{tabularx}
\usepackage{natbib}

\usepackage{amsmath,amsfonts,bm}
\usepackage{xcolor}
\usepackage{fontawesome}

\def\eqref#1{(\ref{#1})}

\def\1{\bm{1}}

\def\vzero{{\bm{0}}}
\def\vone{{\bm{1}}}

\def\vtheta{{\bm{\theta}}}

\def\vphi{{\bm{\phi}}}
\def\vpsi{{\bm{\psi}}}

\def\vx{{\mathbf{x}}}
\def\vy{{\mathbf{y}}}
\def\vz{{\mathbf{z}}}

\def\mA{{\mathbf{A}}}

\def\mI{{\mathbf{I}}}

\def\mS{{\mathbf{S}}}

\def\mX{{\mathbf{X}}}

\def\mZ{{\mathbf{Z}}}

\DeclareMathAlphabet{\mathsfit}{\encodingdefault}{\sfdefault}{m}{sl}
\SetMathAlphabet{\mathsfit}{bold}{\encodingdefault}{\sfdefault}{bx}{n}

\def\sP{{\mathbb{P}}}

\newcommand{\R}{\mathbb{R}}

\newcommand{\abs}[1]{\left\lvert #1 \right\rvert}

\newcommand{\norm}[1]{\left\lVert#1\right\rVert}
\newcommand{\T}{\top}

\usepackage[textsize=tiny, backgroundcolor=blue!30]{todonotes}
\usepackage{listings}
\usepackage{color}

\usepackage{algorithm,algorithmicx,algpseudocode}

\usepackage{subcaption}

\newtheorem{theorem}{Theorem}
\newtheorem{lemma}[theorem]{Lemma}

\newcommand{\BibTeX}{B\kern-.05em{\sc i\kern-.025em b}\kern-.08em\TeX}

\begin{document}


\begin{frontmatter}


\paperid{0046} 


\title{Enhancing Fairness in Autoencoders for Node-Level Graph Anomaly Detection}


\author[A]{\fnms{Shouju}~\snm{Wang}\footnote{Equal contribution.}}
\author[B]{\fnms{Yuchen}~\snm{Song}\footnotemark}
\author[B]{\fnms{Sheng'en}~\snm{Li}\footnotemark} 
\author[B]{\fnms{Dongmian}~\snm{Zou}\thanks{Corresponding Author. Email: dongmian.zou@duke.edu}}

\address[A]{Wuhan University, China}
\address[B]{Duke Kunshan University, China}


\begin{abstract}
Graph anomaly detection (GAD) has become an increasingly important task across various domains. With the rapid development of graph neural networks (GNNs), GAD methods have achieved significant performance improvements. However, fairness considerations in GAD remain largely underexplored. Indeed, GNN-based GAD models can inherit and amplify biases present in training data, potentially leading to unfair outcomes. While existing efforts have focused on developing fair GNNs, most approaches target node classification tasks, where models often rely on simple layer architectures rather than autoencoder-based structures, which are the most widely used architecturs for anomaly detection. To address fairness in autoencoder-based GAD models, we propose \textbf{D}is\textbf{E}ntangled \textbf{C}ounterfactual \textbf{A}dversarial \textbf{F}air (DECAF)-GAD, a framework that alleviates bias while preserving GAD performance. Specifically, we introduce a structural causal model (SCM) to disentangle sensitive attributes from learned representations. Based on this causal framework, we formulate a specialized autoencoder architecture along with a fairness-guided loss function. Through extensive experiments on both synthetic and real-world datasets, we demonstrate that DECAF-GAD not only achieves competitive anomaly detection performance but also significantly enhances fairness metrics compared to baseline GAD methods. Our code is available at \url{https://github.com/Tlhey/decaf_code}.
\end{abstract}

\end{frontmatter}


\section{Introduction}

Graph-structured data has gained significant attention due to its unique ability to represent complex relationships across various domains, such as social networks~\cite{traud2012social}, financial systems~\cite{Wang2021ARO}, and recommendation systems~\cite{Fan2019GraphNN}. 
In these domains, graph anomaly detection (GAD) stands out as an important task~\cite{gad}, as anomalies often signal critical irregularities such as fake accounts, fraudulent transactions, or manipulated user interactions. This paper focuses on node-level GAD, which aims to identify anomalous nodes based on their features and connectivity patterns. Node-level anomalies are particularly challenging due to the inherent class imbalance, where anomalous nodes are significantly outnumbered by normal nodes.

Various approaches have been developed to tackle node-level GAD, such as methods based on node representation learning, contrastive learning, and adversarial learning~\cite{10.1145/3570906}.
Among these, node representation learning has become particularly popular due to its ability to learn informative embeddings and detect anomalies based on reconstruction error~\cite{gad}. 
This approach is especially effective in the unsupervised setting, where only the input graph is available and no labels are provided during training. 
A typical mechanism of learning node representations involves employing an autoencoder that embeds the input graph data into a latent space and subsequently reconstructing it. Anomalies are identified based on  reconstruction error, where poorly reconstructed nodes are flagged as potential anomalies. 

Despite their effectiveness, autoencoder-based models, particularly those leveraging graph neural networks (GNNs), raise significant fairness concerns.
Indeed, real-world graph data often contain inherent attribute and structural biases~\cite{22GEAR}. Attribute bias arises when demographic groups exhibit different statistical distributions~\cite{agarwal2021towards}, while structural bias occurs when the propagation of attribute values reinforces disparities among groups~\cite{Agarwal2021ProbingGE}. For example, if certain communities in a social network are underrepresented in the training data, a GAD model might disproportionately classify members of these groups as anomalies, leading to unfair outcomes. A key contributor to this issue is representation disparity~\cite{Hashimoto2018FairnessWD}, where minority groups contribute less to the learning objective due to their under-representation in training data. As a result, learned representations fail to adequately capture their characteristics, which leads to lower detection accuracy for these minority groups. The problem is further exacerbated in unsupervised GAD, where the scarcity of labeled anomalies makes bias mitigation even more challenging.

Fairness in GAD has only recently emerged as a  research focus. \citet{Neo2024TowardsFG} provided the first formal study on this problem by benchmarking existing GAD methods and highlighting the need for further investigation into fairness considerations. More recently, \citet{24defend} introduced the first dedicated methodology aimed at addressing fairness in GAD by employing causal disentanglement, a fundamental technique commonly used in fair graph learning frameworks~\cite{2019flexibly, yang2020causalvae, zhu2024fair}. This approach learns two separate node embeddings: one that excludes sensitive attributes and another that retains them, with the expectation that predictions based on the non-sensitive embedding will lead to fairer outcomes. However, their approach lacks explicit causal modeling and does not formally capture the dependencies between sensitive and non-sensitive features.
Structural causal models (SCMs)~\cite{10.5555/1642718}, which are widely used in fair graph learning, provide a principled way to model these relationships. However, a major challenge in fair GAD arises as existing SCM-based methods do not naturally extend to autoencoder-based GAD. In most fairness-aware graph learning settings, autoencoders are used to generate counterfactual examples, which focuses on altering sensitive attributes to study their effects. This design contrasts with the objective in GAD, where the primary goal is to reconstruct normal samples accurately and use reconstruction error to detect anomalies. As a result, these counterfactual-oriented autoencoders are not suitable for GAD tasks.

To address the above issues, we introduce an SCM specifically designed for analyzing autoencoders in the context of GAD. Building on this causal framework, we further propose a novel autoencoder-based architecture that enhances fairness in anomaly detection via disentanglement. We call our model DECAF-GAD, short for \textbf{D}is\textbf{E}ntangled \textbf{C}ounterfactual \textbf{A}dversarial \textbf{F}air Graph Anomaly Detection. Our approach is designed to be plug-and-play, which means that it can be integrated into existing autoencoder-based GAD methods. It incorporates a loss function that balances anomaly detection performance with counterfactual fairness. This loss function integrates disentanglement, reconstruction, and counterfactual components. We summarize our main contributions as follows: 
\begin{itemize} 
\item We present a causal analysis for fair GAD by formulating a structural causal model (SCM) for bias propagation in autoencoders. Based on this analysis, we design a fairness-aware loss function that promotes fairness by disentangling sensitive attributes from learned representations.
\item We propose a plug-and-play, model-agnostic framework, that can be easily integrated into autoencoder-based GAD methods for enhanced fairness. 
\item We conduct extensive experiments on both a synthetic dataset and real-world datasets to demonstrate the effectiveness of DECAF-GAD over various baseline GAD models, which shows that our method consistently improves fairness metrics while maintaining competitive anomaly detection performance.
\end{itemize}


\section{Related Works}
\paragraph{Graph Anomaly Detection (GAD)}

Early approaches to GAD predominantly relied on handcrafted feature engineering or statistical models requiring substantial domain expertise~\cite{valko2008distance,ng2019graph}. While methods like OddBall~\cite{akoglu2010oddball} and Radar~\cite{li2017radar} demonstrated effectiveness in specific scenarios through distribution pattern analysis and residual error computation, they faced limitations in handling complex graph structures and lacked adaptability to diverse real-world applications. With the rise of deep learning, attention shifted toward unsupervised graph autoencoders, which automatically learn low-dimensional node representations and detect anomalies via reconstruction errors~\cite{Sakurada2014AnomalyDU,19DOMINANT}. For instance, DOMINANT~\cite{19DOMINANT} pioneered the integration of graph convolutional networks (GCNs)~\cite{Kipf:2017tc} with deep autoencoders to jointly reconstruct node attributes and graph structure. Subsequent innovations extended this framework through specialized reconstruction mechanisms: DONE~\cite{20AdONE} introduced dual autoencoders for separate structure-attribute reconstruction with adversarial regularization; CoLA~\cite{22Cola} proposed contrastive learning between nodes and their neighborhood substructures for anomaly-aware training; CONAD~\cite{22CONAD} integrated human knowledge through contrastive data augmentation and Siamese graph encoders; GAD-NR~\cite{24GADNR} advanced neighborhood reconstruction encompassing local structures and neighbor attributes for detecting non-cluster anomalies. However, these methods primarily focus on detection accuracy while neglecting potential biases in anomaly scoring, particularly when sensitive attributes correlate with graph structure~\cite{He2023ADAGADAA,2024mirror}.

\paragraph{Fairness in Deep Graph Learning}
Fairness in graph-based machine learning has been widely studied under different fairness notions, including group fairness~\cite{Hardt2016EqualityOO, Zemel2013LearningFR}, individual fairness~\cite{Dwork2011FairnessTA}, counterfactual fairness~\cite{guo2023counterfactual, 2017counterfactual}, and other task-specific fairness notions~\cite{dong2023fairness}.  To tackle unfairness, existing methods generally follow three main strategies \cite{dong2023fairness}: regularization approaches introduce fairness constraints or penalty terms into the training objective to reduce correlations between model predictions and sensitive attributes~\cite{fan2021fair, 2022EDITS}; rebalancing methods focus on adjusting the distribution of the training data~\cite{rahman2019fairwalk} or modifying the graph structure ~\cite{kose2022fair} in order to reduce bias; adversarial learning frameworks encourage a model to learn representations invariant to sensitive attributes by deceiving a discriminator network~\cite{bose2019compositional}. Prominent fair node classification frameworks such as FairGNN \cite{2021FairGNN} combine adversarial debiasing with sensitive-attribute estimators and covariance constraints to achieve unbiased node-level decisions. DegFairGCN \cite{liu2023generalized} introduces a trainable debiasing function that adapts the neighborhood aggregation process at each GNN layer, mitigating potential bias from degree imbalances. While research on fair graph learning focused primarily on node classification and link prediction, recent studies have extended fairness considerations to more diverse tasks such as graph clustering\cite{gkartzios2025fair} and recommendation systems\cite{wu2021learning}. However, fairness in GAD remains an underexplored area, with only a few recent works tackling this problem~\cite{Neo2024TowardsFG,24defend}. Our work differs from these by introducing a generalized SCM specifically designed for autoencoder-based GAD methods.

\paragraph{Casual Disentangle Representation Learning}
Causal analysis plays a critical role in fair learning, providing a structured framework to model causal relationships between variables using SCMs. By explicitly modeling how sensitive attributes influence predictions, SCMs help in identifying and mitigating bias~\cite{22GEAR, 23CAF}. A widely adopted technique in fairness-aware learning is disentangled representation learning, which aims to separate sensitive and non-sensitive components in embeddings. This approach is widely integrated with (variational) autoencoders to enforce separation of latent factors. In graph-based scenarios, causal disentanglement is particularly useful for generating graph counterfactuals. Several recent works leverage causal disentanglement and counterfactual reasoning to enhance fairness in graph learning. NIFTY \cite{agarwal2021towards} introduces fairness by randomly perturbing node attributes and graph structures with layer-wise normalization, aiming to reduce bias propagation in GNNs. GEAR \cite{22GEAR} extends this idea by generating counterfactual subgraphs using GraphVAE and training a Siamese network, but it overlooks the causal links between attributes, which may lead to suboptimal counterfactual quality. CAF \cite{23CAF} improves counterfactual generation by selectively constructing latent counterfactual representations, ensuring that anomaly rarity is preserved and avoiding unrealistic graph modifications. For fair representation learning, several methods introduce disentanglement within graph learning models: CFAD \cite{2023CFAD} employs adversarial autoencoders to separate sensitive attributes from learned embeddings. FairSAD \cite{zhu2024fair} introduces channel masking to minimize the impact of sensitive attributes on model predictions. Recent works \cite{wang2024advancing, Zhang2025Disentangled} formulate supervised counterfactual fair embedding methods. Unlike these methods, our work explicitly models causal relationships for autoencoder-based GAD, leading to a different SCM and consequently a different graph learning framework.

\section{Methodology}
\subsection{Structural causal model for autoencoders}
The primary goal of DECAF-GAD is to block the influence of sensitive attributes in the anomaly detection process. 
To formally represent the relationships among these variables, we employ the SCM to analyse the causal relations.

Fairness analysis in graph learning is often based on the SCM as shown in the left subfigure of Figure~\ref{fig:SCM}. In this conventional SCM, the input graph is denoted by $G$, the sensitive attribute by $S$, and the output by either $G'$ for link prediction tasks or $Y$ for node classification tasks. However, this formulation is too simple to work for autoencoder-based GAD, where latent representations play a crucial role in anomaly detection. To address this limitation, we propose an alternative SCM, illustrated in the right subfigure of Figure~\ref{fig:SCM}. 

\begin{figure}[h]
    \centering
    \includegraphics[width=0.8\linewidth]{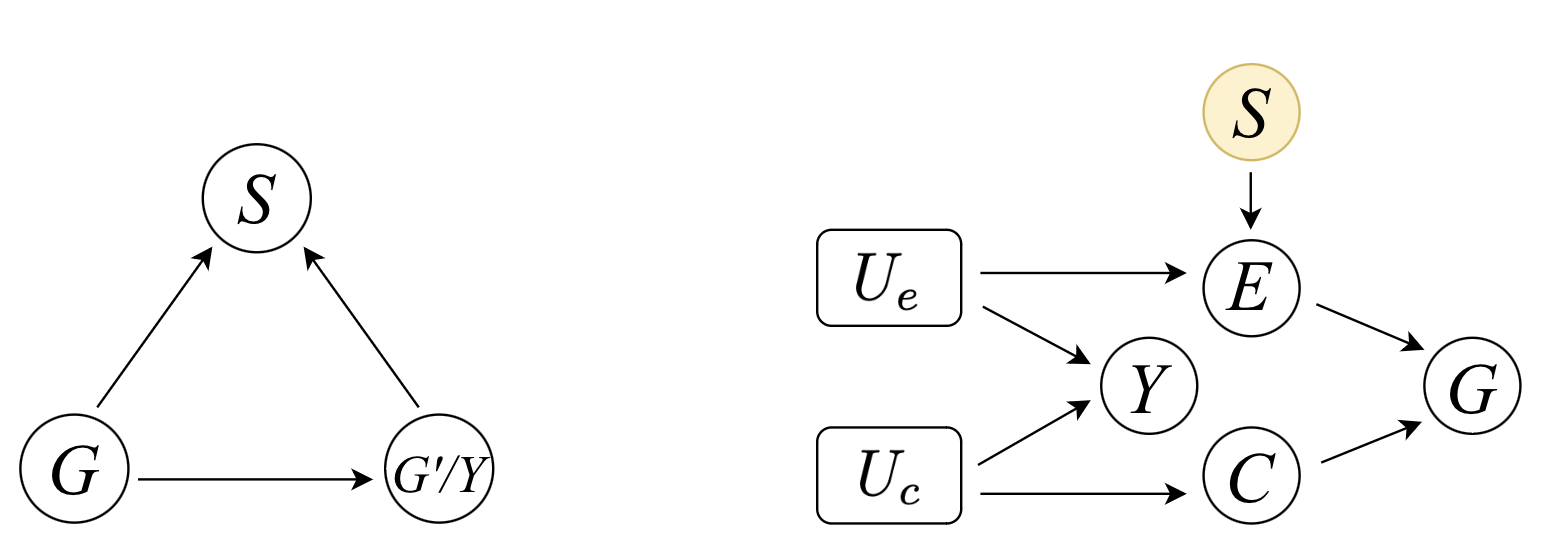}
    \vskip 1em
    \caption{Common SCM for fair learning on graphs (left) and our SCM for fair GAD (right).}
    \vskip 2em
    \label{fig:SCM}
\end{figure}

A key feature of our proposed SCM is the introduction of a content variable $C$ and an environment variable $E$, which together provide a disentangled representation of the graph $G$. Specifically, $C$ captures the content of the graph while excluding sensitive attributes, and $E$ encapsulates the residual information which include sensitive features. 
Furthermore, we introduce latent variables $U_\text{c}$ and $U_\text{e}$, which correspond to latent node representations of the autoencoder for the graph $G$. Both $U_\text{c}$ and $U_\text{e}$ are not directly influenced by $S$. Lastly, $Y$ in this SCM represents the binary label of normal/anomalous for our anomaly detection task. $Y$ is directly influenced by $U_\text{c}$ and $U_\text{e}$ since autoencoder-based GAD methods usually rely on reconstruction from latent node representations.

We first show in the following lemma that $Y$ is independent from $S$ conditioned on $U_\text{c}$ and $U_\text{e}$. Consequently, given the latent node representations of an autoencoder, we can predict $Y$ in a fair manner, without incorporating the sensitive attributes $S$. 

\begin{lemma}[Conditional independence given latent representations]
\label{lemma:cond_independence}
Under the SCM illustrated in the right subfigure of Figure~\ref{fig:SCM}, 
$Y$ is conditionally independent of $S$ given $U_\text{c}$ and $U_\text{e}$, namely:
\[
Y \perp\!\!\!\perp S \mid U_\text{c}, U_\text{e}.
\]
\end{lemma}

\begin{proof}
We prove the lemma by applying the standard d-separation analysis \cite{10.5555/1642718} for SCM. In our SCM, there are two paths connection $S$ and $Y$. The first path is given by $S \to E \leftarrow U_\text{e} \to Y$. Since it has a fork structure at $U_\text{e}$, meeting tail-to-tail, this path is blocked by $U_\text{e}$. The second path is given by $S \to E \to G \leftarrow C \leftarrow U_\text{c} \to Y$. Since it has a fork structure at $U_\text{c}$, meeting tail-to-tail, this path is blocked by $U_\text{c}$. Therefore, $S$ and $Y$ are d-separated by $\{U_\text{c}, U_\text{e}\}$, which indicates that
\[
Y \perp\!\!\!\perp S \mid U_\text{c}, U_\text{e}.
\]

\end{proof}

Finally, again using d-separation analysis, since paths linking $U_\text{c}$ and $U_\text{e}$ are all colliders, meeting head-to-head, we have the unconditional independence relation $U_\text{c} \perp\!\!\!\perp U_\text{e}$. This is a key motivation for disentangling the latent representations of autoencoder.

\subsection{Details of DECAF-GAD}

\begin{figure*}[t]
    \centering
    \includegraphics[width=0.85\textwidth]{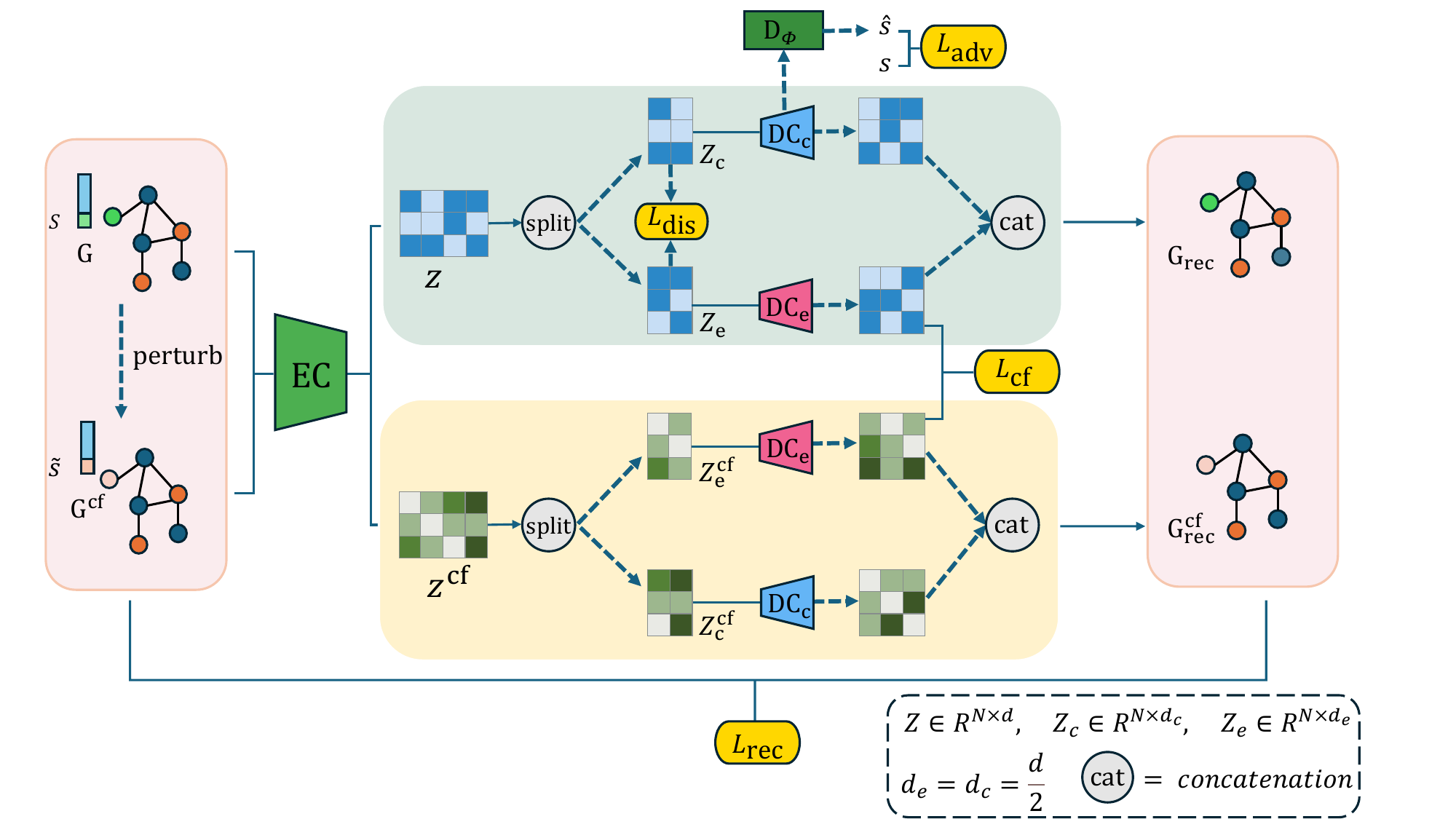}
    \vskip 1emt
    \caption{Architecture of DECAF-GAD. The red shaded region illustrates an autoencoder for the original graph $G$; the blue shaded region illustrates an autoencoder for the counterfactual graph $G^{\text{cf}}$. EC and DC denote the encoder and decoders respectively. This figure also shows the four loss terms $L_\text{rec}$, $L_\text{dis}$, $L_\text{adv}$ and $L_\text{cf}$. For clarify, we omit the structural reconstruction part in our illustration, though our model can accommodate GAD models with such structure.}
    \vskip 2em
    \label{fig:cf_gad_architecture}
\end{figure*}

Building on the above theoretical foundation, we introduce our DECAF-GAD framework. By virtue of the flowchart illustrated in Figure \ref{fig:cf_gad_architecture}, we describe an overview of the model before getting into the details:
Given an input graph $G$, DECAF-GAD generates a counterfactual graph $G^{\text{cf}}$ by flipping the sensitive attributes of $G$. Both $G$ and $G^{\text{cf}}$ are processed by a shared encoder (EC) to obtain latent representations, which are then disentangled into content and environment components, where a disentanglement loss is employed to ensure their separation. More specifically, the latent representations are split into two parts $Z_{\text{e}}$ and $Z_{\text{c}}$, corresponding to $U_\text{e}$ and $U_\text{c}$ in the SCM. On the other hand, by flipping the sensitive variable $S$ to obtain $S^{\text{cf}} = 1 - S$, the encoder produces the latent representations $Z_{\text{e}}^{\text{cf}}$ and $Z_{\text{c}}^{\text{cf}}$ for the counterfactual graph. Adversarial training is employed to enforce the independence of $Z_\text{c}$ from $S$. Separate decoders ($\text{DC}_\text{c}$ and $\text{DC}_\text{e}$) then respectively reconstruct the sensitive and non-sensitive features, producing $\hat{X}_{\text{e}}$, $\hat{X}_{\text{e}}^{\text{cf}}$, $\hat{X}_{\text{c}}$, and $\hat{X}_{\text{c}}^{\text{cf}}$, which are utilized to calculate anomaly scores and enforcing counterfactual constraints. This ensures that $E$ remains consistent across both original and counterfactual graphs, thereby contributing fairly to predictions. Since any autoencoder can be adapted to the architecture of Figure~\ref{fig:cf_gad_architecture}, our DECAF-GAD framework is versatile and can be in integrated with any autoencoder-based GAD method for improving its fairness.

\paragraph{Disentangled representation}
Suppose we are given an attributed graph $G = G(\mX, \mA)$ with $N$ nodes and $D$ features, where $\mX \in \R^{N \times D}$ denotes the feature matrix and $\mA \in \R^{N \times N}$ the adjacency matrix.
To implement the disentangled representation learning approach, let $f_\vtheta$ denote the autoencoder where $\vtheta$ summarizes the parameters.
\begin{equation*}
[ {\mZ_\text{c}}, {\mZ_\text{e}} ] = f_{\vtheta}(\mX, \mA),
\end{equation*}
where $\mZ_\text{c} \in \R^{N \times d_{\text{c}}}$ and $\mZ_\text{e} \in \R^{N \times d_{\text{e}}}$ are the latent representations corresponding to the content and environmental variables. For simplicity we choose $d_{\text{c}} = d_{\text{e}} = d/2$, where $d$ is the latent dimension of the autoencoder.
To ensure effective disentanglement, we need $\mZ_\text{c}$ and $\mZ_\text{e}$ to be dissimilar. We employ a cosine similarity loss due to its simplicity and effectiveness. Namely,
\begin{equation*}
{L}_{\text{dis}}(\vtheta) = \frac{1}{N} \sum_{n=1}^N \abs{\cos(z_\text{c}^n, z_\text{e}^n)},
\end{equation*}
where $z_\text{c}^n$ and $z_\text{e}^n$ are the $n$-th row of $\mZ_\text{c}$ and $\mZ_\text{e}$, respectively; and $\cos(\vx, \vy) := \vx^\T \vy / (\norm{\vx} \norm{\vy})$ for input vectors $\vx$ and $\vy$. 

The latent representations $\mZ_\text{c}$ and $\mZ_\text{e}$ are reconstructed by two separate decoders $g_{\vphi_{\text{c}}}$ and $g_{\vphi_{\text{e}}}$:
\[
\hat{\mX}_\text{c} = g_{\vphi_{\text{c}}}(\mZ_\text{c}); \quad \hat{\mX}_\text{e} = g_{\vphi_{\text{e}}}(\mZ_\text{e}).
\]
The parameters $\vphi_{\text{c}}$ and $g_{\vphi_{\text{e}}}$ are determined by the losses introduced in the subsequent components.

\paragraph{Adversarial learning}
To reinforce the independence of the content representation $\mZ_\text{c}$ from the sensitive attribute $\mS$, we employ an adversarial learning strategy. 
This approach aligns with our SCM in that it prevents the sensitive attribute from influencing the causal pathway $S - G - C - U_{\text{c}} - Y$. 
We introduce an adversarial discriminator, $D_\vpsi$, that attempts to predict the sensitive variable $S$ using $D_\vpsi(\hat{\mX}_\text{c})$. Here, $\vpsi$ summarizes the parameters in the discriminator.
We assume the sensitive variable $S \in \{0,1\}$ is binary as in our datasets. Then, $D_\vpsi$ is trained to maximize the binary cross-entropy loss:
\begin{align*}
L_{\text{adv}}(\vtheta, \vphi_\text{c}, \vpsi) &= - \frac{1}{N} \sum_{n=1}^N \Big( S^n \log D_\vpsi(\vz_{\text{c}}^n) + \\
&\qquad \quad (1-S^n) \log (1-D_\vpsi(\vz_{\text{c}}^n)) \Big),
\end{align*}
where $S^n$ is the sensitive attributes corresponding to the $n$-th node.
Simultaneously, the encoder $f_\vtheta$ is trained to minimize $L_{\text{adv}}$, encouraging the content representation ${\mZ_\text{c}}$ to become invariant to the sensitive attribute $S$. 

\paragraph{Counterfactual regularization}
To explicitly enforce fairness along the path $S - E - U_{\text{e}} - Y$ in our SCM, we introduce a counterfactual regularization term. 
This regularization ensures that the reconstruction remains consistent when sensitive attributes are changed. 
While the adversarial mechanism ensures that the content representation $U_\text{c}$ is fair, the counterfactual promises the quality of $U_\text{e}$ in our model. 

To create a counterfactual graph, we flip the sensitive attribute to produce $\mX^{\text{cf}}$
and then encode this counterfactual graph:
\[
[ {\mZ_\text{c}}^{\text{cf}}, {\mZ_\text{e}^{\text{cf}}} ] = f_{\vtheta}(\mX^{\text{cf}}, \mA),
\]
which is then denoded as
\[
\hat{\mX}_\text{c}^{\text{cf}} = g_{\vphi_{\text{c}}}(\mZ_\text{c}^{\text{cf}}); \quad \hat{\mX}_\text{e}^{\text{cf}} = g_{\vphi_{\text{e}}}(\mZ_\text{e}^{\text{cf}}).
\]
This is in line with the operation of $\text{do}(S = s')$ for the SCM, this operation ensures we are forcibly setting $ S $ to $ s' = 1-s $ while keeping the rest of the node attributes $\mX$ and the structure $\mA$ intact.

To ensure that the latent representation is not affected by sensitive information, we set 
\[
L_\mathrm{cf}(\vtheta, \vphi_\text{e}) = \frac{1}{N} \sum_{n=1}^N \norm{\hat{\mX}_\text{e} - \hat{\mX}_\text{e}^{\text{cf}}}_\text{F}^2,
\]
where $\norm{\cdot}_\text{F}$ denotes the Frobenius norm.

\paragraph{Learning objective}
Suppose an autoencoder-based GAD method has a loss function $L_{\text{rec}}(\vtheta, \vphi_\text{c}, \vphi_\text{e})$, which is commonly taken to be a reconstruction loss. 
The overall learning objective of our DECAF-GAD regularizes this loss function with the aforementioned loss components to balance anomaly detection performance and fairness. More specifically,:
\begin{align*}
L_{\text{total}}(\vtheta, \vphi_\text{c}, \vphi_\text{e}, \vpsi) &=  \lambda_1 L_{\text{rec}}(\vtheta, \vphi_\text{c}, \vphi_\text{e}) + \\
&\quad \lambda_2  L_{\text{dis}}(\vtheta) + \lambda_3  L_{\text{adv}}(\vtheta, \vphi_\text{c}, \vpsi) + \lambda_4  L_{\text{cf}}(\vtheta, \vphi_\text{e}),
\end{align*}
where $\lambda_1$, $\lambda_2$, $\lambda_3$, $\lambda_4$ are hyperparameters. The learning objective is accordingly given by
\begin{equation*}
    \min_{\vtheta, \vphi_\text{c}, \vphi_\text{e}} \max_{\vpsi} ~ L_{\text{total}}(\vtheta, \vphi_\text{c}, \vphi_\text{e}, \vpsi).
\end{equation*}

\paragraph{Complexity Analysis}
We analyze the time complexity of DECAF-GAD. For baseline autoencoder-based GAD models, the computational cost primarily arises from three modules: graph encoding via a GNN, node attribute reconstruction, and graph structure reconstruction. The overall training complexity is 
$O(MDd + NDd + M) = O((M+N)Dd)$,
where \(N\) is the number of nodes, \(M\) is the number of edges, \(D\) is the input feature dimension, and \(d\) is the hidden dimension.

DECAF-GAD introduces additional fairness modules, including disentanglement loss, adversarial loss, and counterfactual graph generation. These components moderately increase the training cost. Specifically, the training complexity becomes $O(2MDd + 2NDd + 2Nd + M)$ due to the doubled forward pass computation (original and counterfactual graphs) and the fairness regularization. We notice that this does not lead to a change in the order of complexity since it still holds that $O(2MDd + 2NDd + 2Nd + M) = O((M+N)Dd)$.
In practice, we observed less than 2x increase in training time and a modest rise in GPU memory.

\section{Experiments}

\subsection{Experiment settings}
\paragraph{Datasets}

We start with datasets designed for benchmarking fair graph learning methods and create datasets for fair GAD by injecting outliers~\cite{22Pygod,huang2023unsupervised} to the graph structures of original datasets. These fair graph learning datasets include three real-world datasets\cite{agarwal2021towards} including German, Bail, and Credit, and one synthetic dataset proposed in \cite{22GEAR}. We follow~\cite{Liu2022BONDBU} and maintain an outlier ratio of 5\% in all experiments. The real-world datasets are described as follows:

\begin{itemize}
    \item German: This dataset represents clients, where each node corresponds to an individual, and edges indicate high similarity in credit account information. The objective is to classify credit risk as either high or low, with gender serving as the sensitive attribute.
    \item Bail: This dataset represents defendants who was released on bail in U.S. state courts between 1990 and 2009. Edges are formed based on similarity in past criminal records and demographics. The classification task is to predict whether a defendant should be granted bail (i.e., unlikely to commit a violent crime if released) or denied bail, with race as the sensitive attribute. 
    \item Credit: This dataset represents credit card users, where nodes represent individuals and edges are established based on similarities in payment behavior. The classification task involves predicting the default payment status, with age as the sensitive attribute.
\end{itemize}

The synthetic dataset is generated using a simple random process, which contains the ground-truth counterfactuals for evaluating counterfactual fairness. On the other hand, the real-world datasets only facilitates evalution of fairness metrics such as equal opportunity (EOO) and demographic parity (DP).
Specifically, in the synthetic dataset \cite{22GEAR}, the graph features are generated as follows:
\[
S^n \sim \text{Bernoulli}(p), \quad \vz^n \sim \mathcal{N}(\vzero, \mI), \quad \vx^n = \operatorname{sample}(\vz^n) + S^n \mathbf{v}, 
\]
where $ p $ is taken to  be $0.4$ following~\cite{22GEAR}, $ \vz^n \in \mathbb{R}^{50} $ is the a latent embedding, and $ \operatorname{sample}(\cdot) $ denotes a sampling operation which randomly selects $ 25 $ dimensions out of the latent embeddings to form the observed features $ \vx^n $ for the $n$-th node.
The graph structure is constructed by assigning the following probability for connecting nodes $n$, $m$: \[
\sP (A_{nm} = 1) = \sigma(\cos(\vz^n, \vz^m) + a \cdot \vone(S^n = S^m)),\]
where $ a = 0.01 $, $\sigma$ is the sigmoid function, and $ \vone(\cdot) $ is an indicator function which outputs 1 when the input statement is true and 0 otherwise.
Finally, the label for the $n$-th node is given by a binary quantization of $\mathbf{w}^\T \vz^n + {\sum_{m \in \operatorname{nei}(n)} S^m}/{(2|\operatorname{nei}(n)|)}$, where $\operatorname{nei}(n)$ denotes the set of nodes adjacent to the $n$-th node.

The statistics of both synthetic and real datasets are summarized in Table~\ref{tab:stats}. To inject the structural outliers, we follow \cite{22Pygod} to create multiple groups of anomalous nodes, each of which contains a fixed number of nodes. For each group, we begin by fully connecting all nodes within the group, forming a dense subgraph. To introduce structural anomalies, we then randomly remove edges, where each edge is dropped with a certain probability. We keep an outlier ratio of 5\% in all the experiments. 

\begin{table}[t]
\centering
\caption{Statistics of the graph datasets.}
\label{tab:stats}
\begin{tabular}{lccccc}
\toprule
Dataset & \# Nodes & \# Edges & \# Attributes & Sens.  \\
\midrule
Synthetic& 2000 & 5090 & 50 & S \\
German & 1,000 & 22,242 & 27 & gender  \\
Bail & 18,876 & 321,308 & 18 & race  \\
Credit  & 30,000 & 1,436,858 & 13 & age  \\
\bottomrule
\end{tabular}
\end{table}

\begin{table*}[t] 
\centering
\caption{Performance comparison on real-world datasets with structural outliers. $\uparrow$ denotes that higher values indicate better performance, $\downarrow$ denotes that lower values indicate better performance. The better results between baseline and DECAF-GAD for each scenario are highlighted in \textbf{bold}. GADNR showed numerical instability for the Credit dataset and is not reported.}
\label{tab:main_results_real_main}
\vskip 1em
\begin{tabular}{llccccc}
\toprule
Dataset & Method & Accuracy (\%) $\uparrow$ & F1-score (\%) $\uparrow$ & AUROC (\%) $\uparrow$ & $\Delta_{\text{EOO}}$ (\%) $\downarrow$ & $\Delta_{\text{DP}}$ (\%) $\downarrow$ \\
\midrule
\multirow{6}{*}{Bail} 
 & DOMINANT        & \textbf{91.90 $\pm$ 0.04} & \textbf{19.04 $\pm$ 0.43} & \textbf{57.38 $\pm$ 0.23} & 2.64 $\pm$ 0.85 & 0.84 $\pm$ 0.05 \\
 & DECAF-DOMINANT  & 91.49 $\pm$ 0.35 & 14.95 $\pm$ 3.54 & 55.23 $\pm$ 1.86 & \textbf{2.41 $\pm$ 1.33} & \textbf{0.61 $\pm$ 0.30} \\
 \cmidrule(l){2-7}
 & DONE            & 91.12 $\pm$ 0.06 & 11.23 $\pm$ 0.59 & 53.28 $\pm$ 0.31 & 1.09 $\pm$ 0.94 & 0.79 $\pm$ 0.34 \\
 & DECAF-DONE      & \textbf{91.32 $\pm$ 0.08} & \textbf{13.28 $\pm$ 0.79} & \textbf{54.35 $\pm$ 0.41} & \textbf{0.95 $\pm$ 0.86} & \textbf{0.63 $\pm$ 0.33} \\
 \cmidrule(l){2-7}
 & GADNR           & 92.11 $\pm$ 0.49 & \textbf{21.16 $\pm$ 4.91} & \textbf{58.50 $\pm$ 2.58} & 5.76 $\pm$ 2.21 & 0.75 $\pm$ 0.35 \\
 & DECAF-GADNR     & \textbf{93.57 $\pm$ 1.50} & 10.71 $\pm$ 11.65 & 54.32 $\pm$ 4.94 & \textbf{2.58 $\pm$ 2.99} & \textbf{0.41 $\pm$ 0.45} \\
\midrule
\multirow{6}{*}{German} 
 & DOMINANT        & \textbf{88.74 $\pm$ 0.38} & \textbf{19.57 $\pm$ 2.71} & \textbf{55.62 $\pm$ 1.16} & 12.07 $\pm$ 2.89 & \textbf{1.96 $\pm$ 0.66} \\
 & DECAF-DOMINANT  & 87.64 $\pm$ 0.69 & 11.71 $\pm$ 4.94 & 51.34 $\pm$ 1.34 & \textbf{3.70 $\pm$ 2.66} & \textbf{1.96 $\pm$ 0.30} \\
 \cmidrule(l){2-7}
 & DONE            & \textbf{87.40 $\pm$ 0.00}& \textbf{10.00 $\pm$ 0.00} & \textbf{51.53 $\pm$ 0.00} & 2.07 $\pm$ 0.00 & \textbf{1.64 $\pm$ 0.00} \\
 & DECAF-DONE      & \textbf{87.40 $\pm$ 0.90} & \textbf{10.00 $\pm$ 0.63} & 51.52 $\pm$ 0.27 & \textbf{2.04 $\pm$ 0.77} & \textbf{1.64 $\pm$ 0.59} \\
 \cmidrule(l){2-7}
 & GADNR           & 88.64 $\pm$ 1.53 & 18.86 $\pm$ 10.91 & 55.31 $\pm$ 4.66 & \textbf{11.20 $\pm$ 6.66} & \textbf{1.64 $\pm$ 0.63} \\
 & DECAF-GADNR     & \textbf{90.00 $\pm$ 0.00} & \textbf{28.57 $\pm$ 0.00} & \textbf{59.46 $\pm$ 0.00} & 15.95 $\pm$ 0.00 & \textbf{1.64 $\pm$ 0.00} \\
\midrule
\multirow{4}{*}{Credit} 
 & DOMINANT        & \textbf{92.17 $\pm$ 0.11} & \textbf{22.64 $\pm$ 1.13} & \textbf{59.15 $\pm$ 0.59} & 12.75 $\pm$ 0.94 & 3.04 $\pm$ 0.75 \\
 & DECAF-DOMINANT  & 91.03 $\pm$ 0.68 & 11.33 $\pm$ 6.70 & 53.27 $\pm$ 3.49 & \textbf{4.90 $\pm$ 3.69} & \textbf{1.89 $\pm$ 0.90} \\
 \cmidrule(l){2-7}
 & DONE            & \textbf{97.30 $\pm$ 0.11} & \textbf{73.33 $\pm$ 1.07} & \textbf{85.56 $\pm$ 0.56} & 28.88 $\pm$ 1.23 & 2.68 $\pm$ 0.26 \\
 & DECAF-DONE      & 94.04 $\pm$ 0.24 & 41.11 $\pm$ 2.40 & 68.78 $\pm$ 1.25 & \textbf{15.11 $\pm$ 2.84} & \textbf{2.06 $\pm$ 0.22} \\
\bottomrule
\end{tabular}
\vskip 1em
\end{table*}

\begin{table*}[t]
\centering
\caption{Performance comparison on the synthetic dataset with structural outliers. $\uparrow$ denotes that higher values indicate better performance, $\downarrow$ denotes that lower values indicate better performance. The better results between baseline and DECAF-GAD for each scenario are highlighted in \textbf{bold}.}
\label{tab:main_results_synt}
\vskip 1em
\begin{tabular}{lcccccc}
\toprule
Method & Accuracy (\%) $\uparrow$ & F1-score (\%) $\uparrow$ & AUROC (\%) $\uparrow$ & $\Delta_{\text{EOO}}$ (\%) $\downarrow$ & $\Delta_{\text{DP}}$ (\%) $\downarrow$ & $\Delta_{\text{CF}}$ (\%) $\downarrow$ \\
\midrule
 DOMINANT & 92.36 $\pm$ 0.28 & 24.36 $\pm$ 2.74 & 60.07 $\pm$ 1.43 & 8.00 $\pm$ 8.22 & 5.89 $\pm$ 2.29 & 7.32 $\pm$ 0.35 \\
  DECAF-DOMINANT   & \textbf{92.80 $\pm$ 0.12} & \textbf{28.71 $\pm$ 1.17} & \textbf{62.35 $\pm$ 0.61} & \textbf{3.67 $\pm$ 1.63} & \textbf{4.75 $\pm$ 0.64} & \textbf{6.71 $\pm$ 0.18} \\
\midrule

  DONE & \textbf{96.62 $\pm$ 0.35} & \textbf{66.53 $\pm$ 3.42} & \textbf{82.08 $\pm$ 1.79} & 5.16 $\pm$ 5.55 & 0.63 $\pm$ 0.29 & 4.09 $\pm$ 0.33 \\
  DECAF-DONE   & 96.49 $\pm$ 0.50 & 65.25 $\pm$ 4.92 & 81.41 $\pm$ 2.57 & \textbf{3.88 $\pm$ 3.98} & \textbf{0.58 $\pm$ 0.34} & \textbf{3.41 $\pm$ 0.50} \\
\midrule

 GADNR & 93.81 $\pm$ 1.93 & 38.71 $\pm$ 19.14 & 67.56 $\pm$ 9.99 & 15.52 $\pm$ 7.53 & 3.73 $\pm$ 2.17 & 4.69 $\pm$ 3.33 \\
 DECAF-GADNR   & \textbf{95.40 $\pm$ 0.33} & \textbf{54.46 $\pm$ 3.31} & \textbf{75.78 $\pm$ 1.73} & \textbf{15.46 $\pm$ 2.95} & \textbf{2.41 $\pm$ 0.39} & \textbf{0.57 $\pm$ 0.42} \\
\bottomrule
\end{tabular}
\vskip 1em
\end{table*}

\paragraph{Baseline and experiment setup} To demonstrate the effectiveness and flexibility of DECAF-GAD, we integrate it with three popular autoencoder-based baseline GAD methods: DOMINANT~\cite{19DOMINANT}, DONE~\cite{adone}, and GADNR~\cite{24GADNR}. We refer to the integrated models as DECAF-DOMINANT, DECAF-DONE, and DECAF-GADNR, respectively. We evaluate each DECAF-GAD model against its original baseline to assess the effectiveness of our debiasing approach. All experiments are conducted $10$ times randomly, and we report the average performance metrics along with standard errors to ensure statistical reliability.

\paragraph{Implementation Details} 

We implement the baseline methods using the PyGOD library \cite{22Pygod} where recommended settings are used for hyperparameters including hidden dimensions and number of layers, while optimizing learning rate, dropout rate, and weight decay using Optuna \cite{19optuna}. The search ranges are as follows: learning rate \( [1e^{-5}, 1e^{-1}] \), dropout rate \( [1e^{-6}, 1] \), and weight decay \( [1e^{-6}, 1] \), all using a logarithmic search space. We also use Optuna to search for the hyperparameters for \( \lambda_1, \lambda_2, \lambda_3, \lambda_4 \) in DECAF-GAD. All experiments are conducted on a Linux server with GeForce RTX 4090 GPUs (24 GB RAM), where each experiment uses a single GPU.

\paragraph{Evaluation metric} To assess anomaly detection accuracy, we adopt commonly used metrics including classification accuracy, F1-score, and AUC-ROC score. To assess fairness, we adopt commonly used metrics including Equal of Opportunity (EOO) and Demographic Parity (DP). In addition, since ground-truth counterfactual data is  available in the synthetic dataset, we also report the counterfactual fairness (CF) metric for the synthetic data. Specifically, these metrics are given by
\begin{align*}
    \Delta_\text{EOO} &= \left| \sP(\hat{Y} = 1 \mid Y = y, S = 0) - \sP(\hat{Y} = 1 \mid Y = y, S = 1) \right|, \\
    \Delta_\text{DP} &= \left| \sP(\hat{Y} = 1 \mid S = 0) - \sP(\hat{Y} = 1 \mid S = 1) \right|, \\
    \Delta_\text{CF} &= \left| \sP(\hat{Y}=1 \mid {S \leftarrow 0}) - \sP(\hat{Y}=1 \mid {S \leftarrow 1}) \right|,
\end{align*}
where $\leftarrow$ stands for the ``do'' operation in counterfactual analysis. A smaller value in these metrics indicates better fairness.

\subsection{Results}

Tables \ref{tab:main_results_real_main} and \ref{tab:main_results_synt} summarize the performance of our DECAF-GAD framework compared to baseline methods on real-world and synthetic datasets, respectively. We remark that GADNR showed numerical instability for the Credit dataset and thus it is not reported. For a clear comparison, we also present bar plots for comparing the baseline methods and the DECAF versions for the synthetic data.

For the three real-world datasets, DECAF-GAD consistently outperforms the baseline methods in terms of fairness metrics $\Delta_{\text{EOO}}$ and $\Delta_{\text{DP}}$. As a tradeoff for improved fairness, the accuracy is slightly reduced in some scenarios; however, some scenarios still see an improvement in accuracy. 

For the synthetic dataset, DECAF-GAD demonstrates consistently superior performance across all fairness metrics including not only $\Delta_{\text{EOO}}$ and $\Delta_{\text{DP}}$, but also $\Delta_{\text{CF}}$. More encouragingly, it also achieves similar or even better accuracies in all scenarios. This superior performance can be attributed to the fact that the design of our framework has a clear causal structure. 

In addition to the structural perturbation setting, we also consider injecting the contextual outliers on node features. 

\paragraph{Ablation study}
To assess the effectiveness of the adversarial learning and counterfactual loss in the design of our method, we perform an ablation study on the Bail dataset. 

Figure \ref{fig:ablation2} presents the results for the full model, the model without adversarial learning, the model without counterfactual regularization, and the model with both components removed. The results clearly indicate that the combination of adversarial loss and counterfactual regularization substantially enhances the fairness performance of the model.

\begin{figure}[t]
    \centering
        \includegraphics[width=\linewidth]{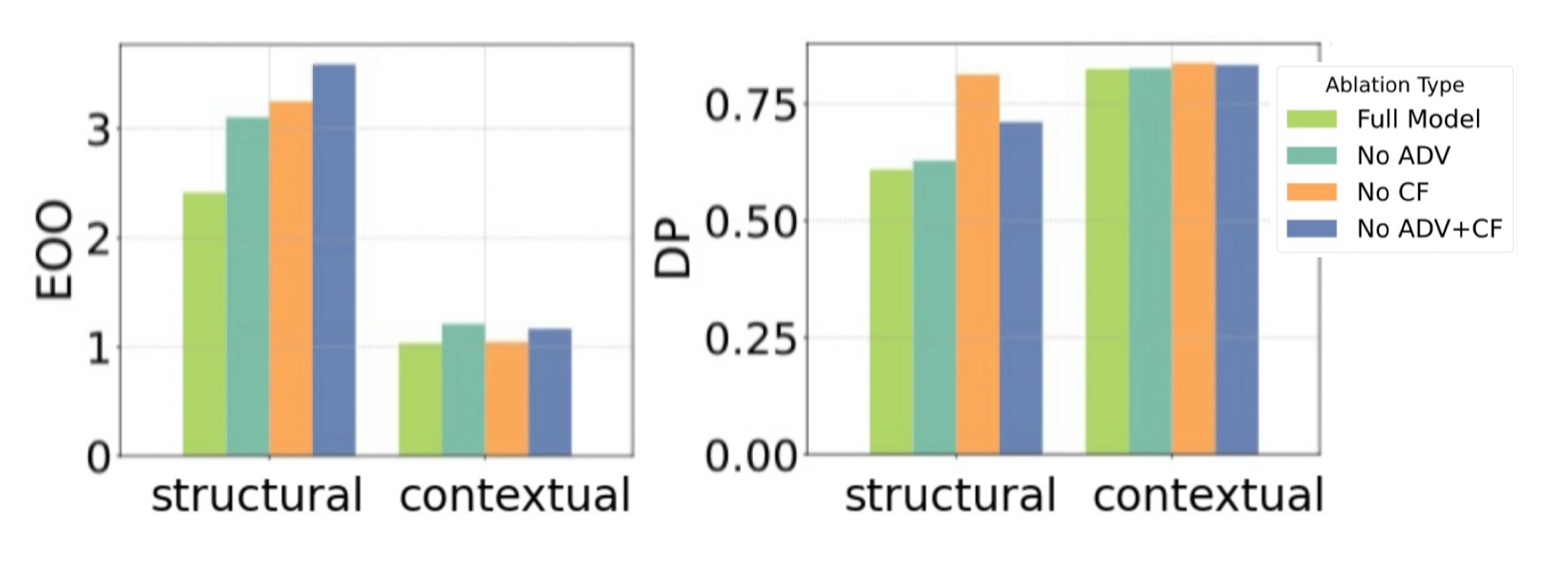}
        \vskip 1em
        \caption{Ablation study of adversarial learning and counterfactual fairness components of our method. We test full model versus removing these components and report the $\Delta_{\text{EOO}}$ and $\Delta_{\text{DP}}$ metrics on the Bail dataset.}
        \vskip 2em
        \label{fig:ablation2}
\end{figure}

\paragraph{Sensitivity analysis}
We conduct a sensitivity analysis on the hyperparameters $\lambda_1$, $\lambda_2$ and $\lambda_3$. Specifically, we examine their impact on the counterfactual fairness metric  $\Delta_\text{CF}$ for the synthetic dataset using the DECAF-DOMINANT method. Figure \ref{fig:sensitivity} presents the results. Each subfigure shows the heatmap of the $\Delta_\text{CF}$ score as two of the three hyperparameters vary. Our analysis reveals that the performance of our model remains relatively stable, with the range of $\Delta_\text{CF}$ values across all subfigures spanning less than 2\%, which indicates robustness to hyperparameter selection.


\begin{figure*}[t] 
    \centering
    \begin{subfigure}{0.3\linewidth} 
        \centering
        \includegraphics[width=\linewidth]{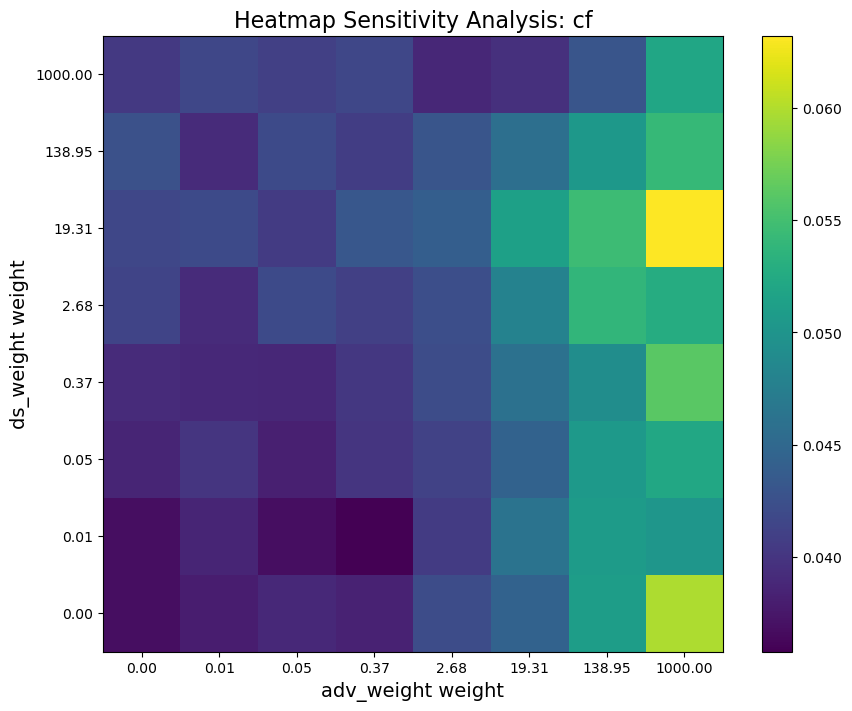}
        \caption{$\lambda_1$ vs. $\lambda_2$}
    \end{subfigure}
    \hfill
    \begin{subfigure}{0.3\linewidth} 
        \centering
        \includegraphics[width=\linewidth]{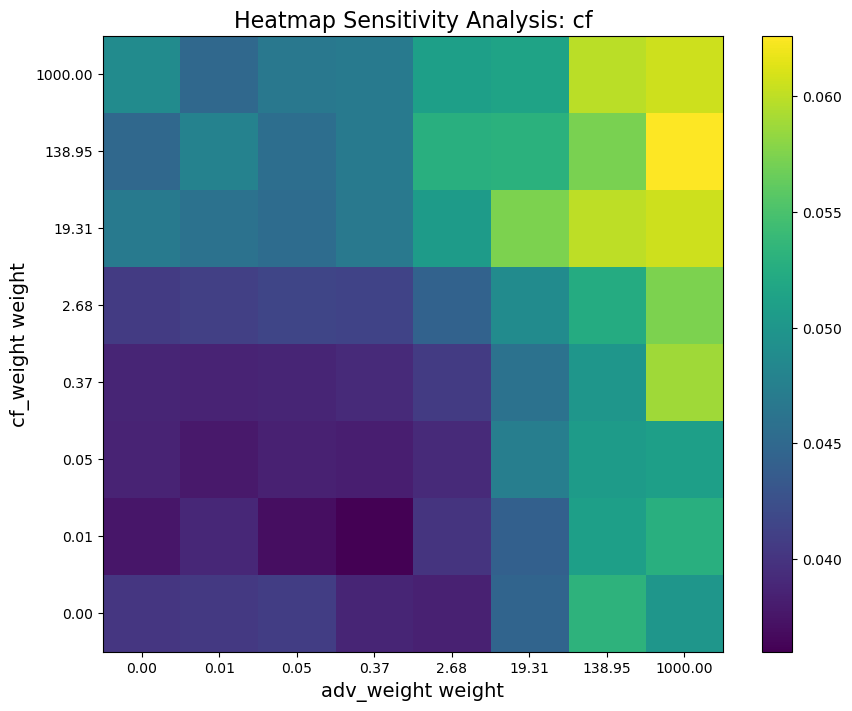} 
        \caption{$\lambda_2$ vs. $\lambda_3$}
    \end{subfigure}
    \hfill
    \begin{subfigure}{0.3\linewidth} 
        \centering
        \includegraphics[width=\linewidth]{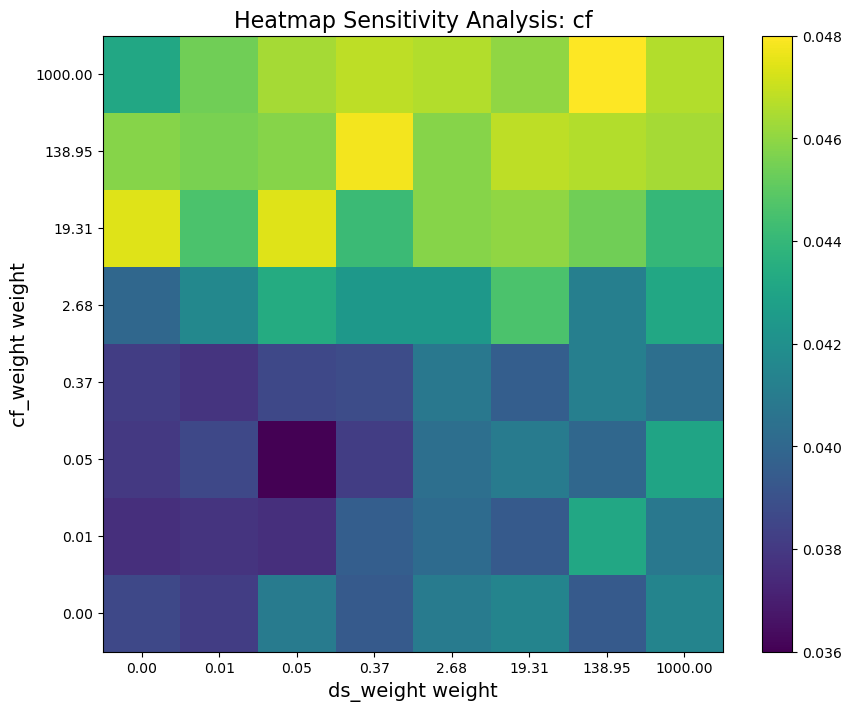} 
        \caption{$\lambda_1$ vs. $\lambda_3$}
    \end{subfigure}
    \vskip 1em
    \caption{Sensitivity analysis for $\Delta_\text{CF}$ scores.}
    \vskip 2em
    \label{fig:sensitivity}
    
\end{figure*}

\paragraph{Correlation study}

\begin{figure}[t]
    \centering
    \includegraphics[width=.6\linewidth]{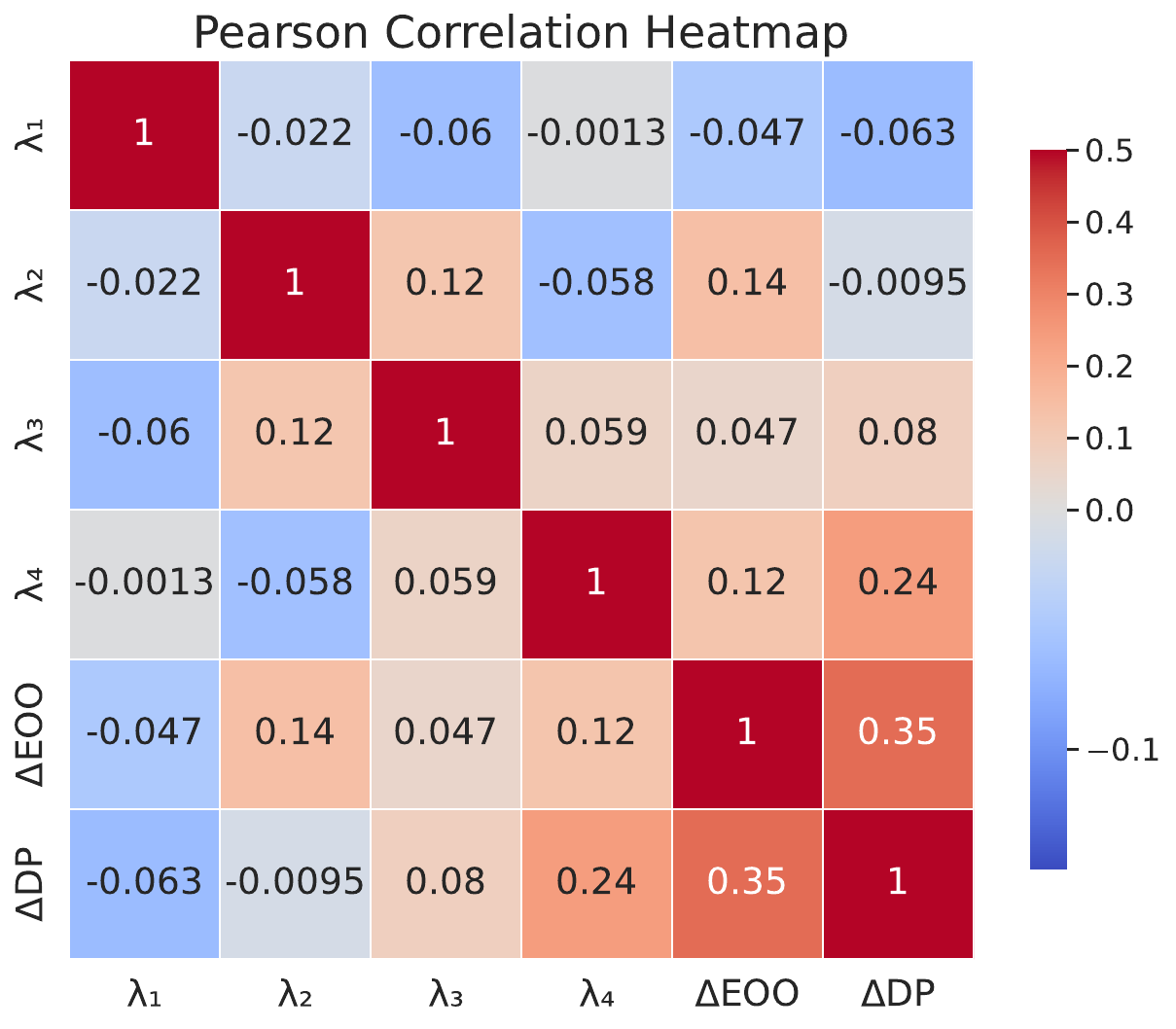}
    \vskip 1em
    \caption{Correlation between $\lambda_1$, $\lambda_2$, $\lambda_3$, $\lambda_4$, $\Delta_{\text{EOO}}$ and $\Delta_{\text{DP}}$ on the Bail dataset for DECAF-DOMINANT.}
    \vskip 2em
    \label{fig:correlation}
\end{figure}

In this section, we analyze the correlation between: hyperparameters $\lambda_1$, $\lambda_2$, $\lambda_3$, $\lambda_4$; and fairness metrics including DP and EOO. We run 100 trials on the Bail dataset using DECAF-DOMINANT in order to explore these correlations.

The Pearson correlation heatmap is presented in Figure~\ref{fig:correlation}. Clearly, $\lambda_2$, $\lambda_3$ and $\lambda_4$ generally have positive correlations with both DP and EOO. This indicates that the disentanglement design, adversarial learning strategy, and counterfactual fairness constraint all contribute to improved fairness. 
On the other hand, $\lambda_1$ exhibits negative correlations with both DP and EOO, indicating that prioritizing reconstruction can adversely impact fairness. This suggests a trade-off between fairness and accuracy.

\paragraph{Visualization analysis}
In Figure \ref{fig:tsne_z}, we present the two-dimensional t-SNE visualizations \cite{van2008visualizing} for the four distinct embeddings: $\mathbf{Z}_\text{c}$, $\mathbf{Z}_\text{e}$, $\mathbf{Z}_\text{c}^{\mathrm{cf}}$, and $\mathbf{Z}_\text{e}^{\mathrm{cf}}$. These are latent representations from the autoencoder of DECAF-DOMINANT. 

The resulting clusters show minimal overlap between original and counterfactual representations, which indicates that sensitive attributes are successfully ``flipped'' in the latent space. Moreover, the content components are well-separated from the environment components, which indicates that the strategy of disentanglement is successfully executed.

In Figure~\ref{fig:tsne_sy}, we show t-SNE presentation of the concatenation of $\mathbf{Z}_\text{c}$ and $\mathbf{Z}_\text{e}$. We color the data points according to both the sensitive attributes and the labels. We observe that there is no clear clustering according to the sensitive labels, which intuitively indicates fairness of our method. However, there is also no clear clustering according to the labels, which showcases the tradeoff between fairness and GAD performance.

\section{Conclusion and Future Work}
In this work, we introduced DECAF-GAD, a novel framework designed to address the underexplored issue of fair GAD. By analyzing an SCM, our method integrates disentanglement, adversarial learning, and counterfactual regularization, into existing autoencoder-based GAD methods. Our extensive experiments on both real-world and synthetic datasets demonstrate that DECAF-GAD consistently improves fairness metrics across various baseline GAD models, showcasing the effectiveness, robustness, and adaptability of our framework. These results highlight the potential of causal modeling as well as disentanglement techniques in improving fairness in GAD problems.
While DECAF-GAD is designed to integrate seamlessly with existing autoencoder-based GAD models , its core principles can be extended to other GAD approaches.

Future work includes the following directions. First, our framework currently focuses on autoencoder-based GAD; extending it to other anomaly detection approaches (e.g., generative modeling-based approaches) could enhance its applicability. Second, while our counterfactual fairness formulation relies on a fixed SCM that we propose, it is possible to consider learning causal graphs so that the SCM adapts to more complex real-world scenarios. Moreover, an important direction for future work is to expand fair GAD beyond group fairness metrics and counterfactual fairness metrics and explore the different setting of individual fairness in GAD. Finally, we could consider other types of graph data, such as dynamic graphs that model evolving real-world systems.

\begin{figure}[t]
    \centering
    \includegraphics[width=.7\linewidth]{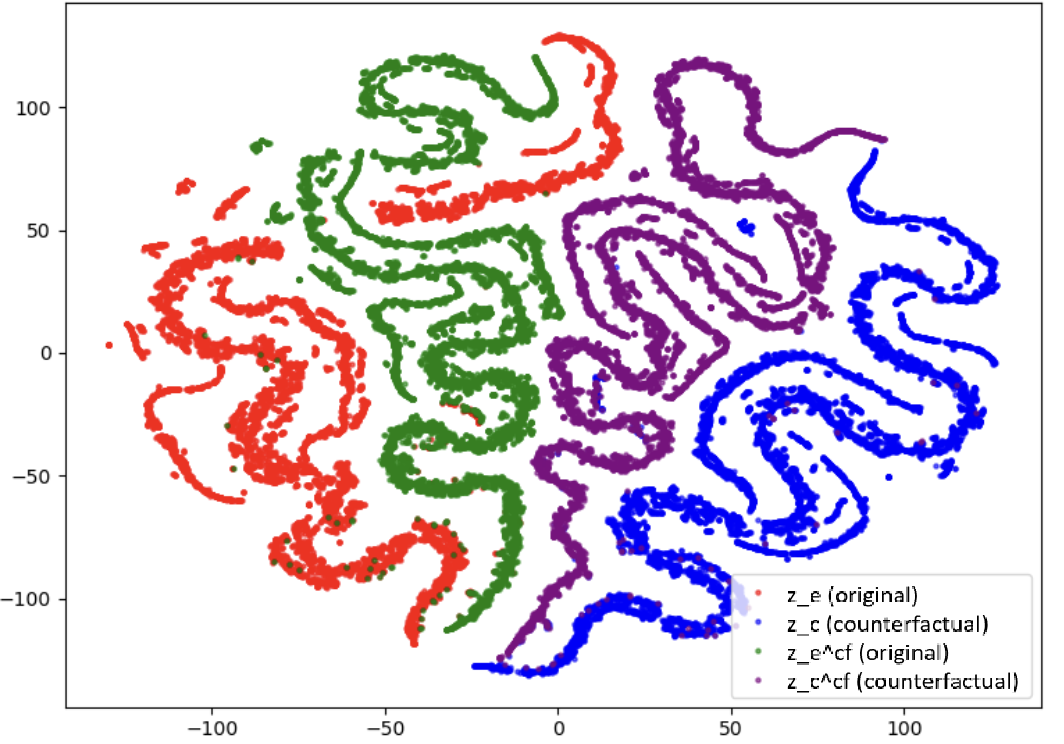}
    \vskip 1em
    \caption{t-SNE visualization of disentangled latent representations. The colors indicate four different latent variables.}
    \vskip 2em
    \label{fig:tsne_z}
\end{figure}

\begin{figure}[t]
    \centering
    \includegraphics[width=.7\linewidth]{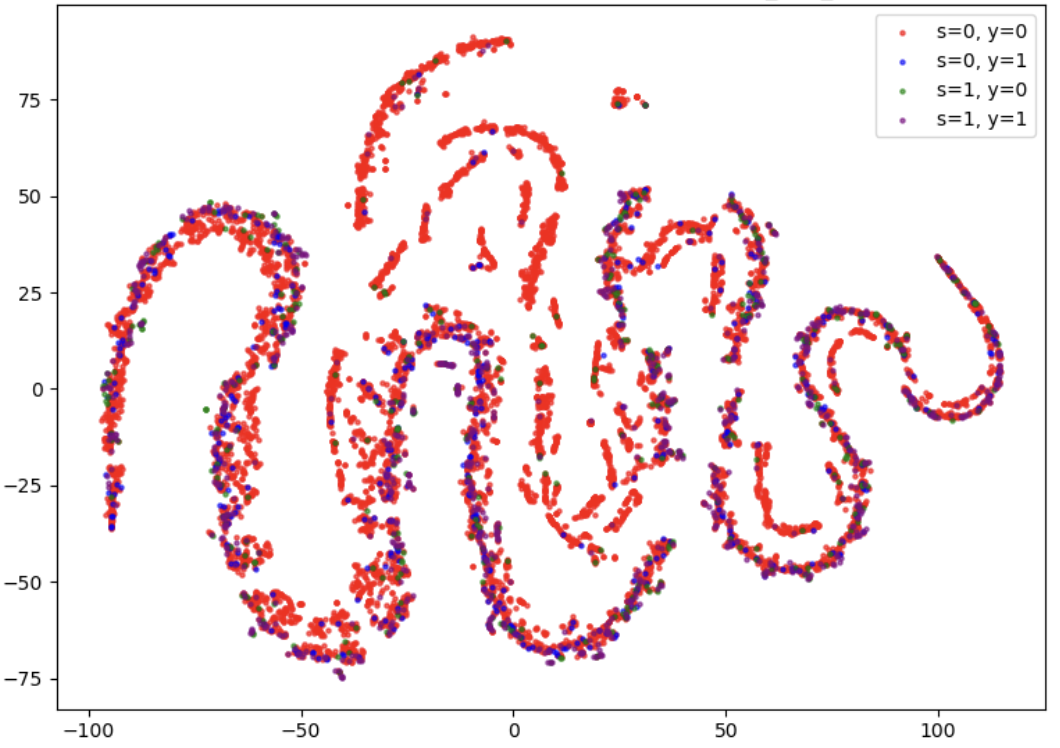}
    \vskip 1em
    \caption{t-SNE visualization of concatenated latent representations. The colors indicate different sensitive attributes and different labels.}
    \vskip 2em
    \label{fig:tsne_sy}
\end{figure}



\begin{ack}
Part of this work was conducted while Shouju Wang was visiting Duke Kunshan University as an exchange student. Yuchen Song's contribution was partially supported by the Duke Kunshan University Summer Research Scholar Program. We also thank the anonymous reviewers for their insightful comments and constructive feedback. 
\end{ack}



\bibliography{mybibfile}

\begin{thebibliography}{53}
\providecommand{\natexlab}[1]{#1}
\providecommand{\url}[1]{\texttt{#1}}
\expandafter\ifx\csname urlstyle\endcsname\relax
  \providecommand{\doi}[1]{doi: #1}\else
  \providecommand{\doi}{doi: \begingroup \urlstyle{rm}\Url}\fi

\bibitem[Agarwal et~al.(2021{\natexlab{a}})Agarwal, Lakkaraju, and Zitnik]{agarwal2021towards}
C.~Agarwal, H.~Lakkaraju, and M.~Zitnik.
\newblock Towards a unified framework for fair and stable graph representation learning.
\newblock In \emph{Uncertainty in Artificial Intelligence}, pages 2114--2124. PMLR, 2021{\natexlab{a}}.

\bibitem[Agarwal et~al.(2021{\natexlab{b}})Agarwal, Zitnik, and Lakkaraju]{Agarwal2021ProbingGE}
C.~Agarwal, M.~Zitnik, and H.~Lakkaraju.
\newblock Probing {GNN} explainers: A rigorous theoretical and empirical analysis of {GNN} explanation methods.
\newblock In \emph{AISTATS}, 2021{\natexlab{b}}.

\bibitem[Akiba et~al.(2019)Akiba, Sano, Yanase, Ohta, and Koyama]{19optuna}
T.~Akiba, S.~Sano, T.~Yanase, T.~Ohta, and M.~Koyama.
\newblock Optuna: A next-generation hyperparameter optimization framework.
\newblock In \emph{International Conference on Knowledge Discovery and Data Mining}, 2019.

\bibitem[Akoglu et~al.(2010)Akoglu, McGlohon, and Faloutsos]{akoglu2010oddball}
L.~Akoglu, M.~McGlohon, and C.~Faloutsos.
\newblock Oddball: Spotting anomalies in weighted graphs.
\newblock In \emph{PAKDD}, 2010.

\bibitem[Bandyopadhyay et~al.(2020{\natexlab{a}})Bandyopadhyay, Lokesh, Vivek, and Murty]{adone}
S.~Bandyopadhyay, N.~Lokesh, S.~V. Vivek, and M.~N. Murty.
\newblock Outlier resistant unsupervised deep architectures for attributed network embedding.
\newblock \emph{WSDM}, 2020{\natexlab{a}}.

\bibitem[Bandyopadhyay et~al.(2020{\natexlab{b}})Bandyopadhyay, N, Vivek, and Murty]{20AdONE}
S.~Bandyopadhyay, L.~N, S.~V. Vivek, and M.~N. Murty.
\newblock Outlier resistant unsupervised deep architectures for attributed network embedding.
\newblock In \emph{WSDM}, 2020{\natexlab{b}}.

\bibitem[Bose and Hamilton(2019)]{bose2019compositional}
A.~Bose and W.~Hamilton.
\newblock Compositional fairness constraints for graph embeddings.
\newblock In \emph{International Conference on Machine Learning}, 2019.

\bibitem[Chang et~al.(2024)Chang, Liu, Yu, and Yu]{24defend}
W.~Chang, K.~Liu, P.~S. Yu, and J.~Yu.
\newblock Enhancing fairness in unsupervised graph anomaly detection through disentanglement.
\newblock \emph{arXiv: 2406.00987}, 2024.

\bibitem[Creager et~al.(2019)Creager, Madras, Jacobsen, Weis, Swersky, Pitassi, and Zemel]{2019flexibly}
E.~Creager, D.~Madras, J.-H. Jacobsen, M.~A. Weis, K.~Swersky, T.~Pitassi, and R.~Zemel.
\newblock Flexibly fair representation learning by disentanglement.
\newblock In \emph{International Conference on Machine Learning}, 2019.

\bibitem[Dai et~al.(2021)Dai, Wang, and Zhu]{2021FairGNN}
E.~Dai, W.~Wang, and X.~Zhu.
\newblock {FairGNN}: Fairness-aware graph neural networks.
\newblock In \emph{AAAI Conference on Artificial Intelligence}, 2021.

\bibitem[Ding et~al.(2019)Ding, Li, Bhanushali, and Liu]{19DOMINANT}
K.~Ding, J.~Li, R.~Bhanushali, and H.~Liu.
\newblock Deep anomaly detection on attributed networks.
\newblock In \emph{Proceedings of the 2019 SIAM international conference on data mining}, pages 594--602. SIAM, 2019.

\bibitem[Dong et~al.(2022)Dong, Wang, and Sun]{2022EDITS}
Y.~Dong, Z.~Wang, and Y.~Sun.
\newblock Edits: Pre-processing for fair graph learning.
\newblock In \emph{International Conference on Knowledge Discovery and Data Mining}, 2022.

\bibitem[Dong et~al.(2023)Dong, Ma, Wang, Chen, and Li]{dong2023fairness}
Y.~Dong, J.~Ma, S.~Wang, C.~Chen, and J.~Li.
\newblock Fairness in graph mining: A survey.
\newblock \emph{IEEE Transactions on Knowledge and Data Engineering}, 35\penalty0 (10):\penalty0 10583--10602, 2023.

\bibitem[Dwork et~al.(2012)Dwork, Hardt, Pitassi, Reingold, and Zemel]{Dwork2011FairnessTA}
C.~Dwork, M.~Hardt, T.~Pitassi, O.~Reingold, and R.~Zemel.
\newblock Fairness through awareness.
\newblock In \emph{Proceedings of the 3rd innovations in theoretical computer science conference}, pages 214--226, 2012.

\bibitem[Fan et~al.(2019)Fan, Ma, Li, He, Zhao, Tang, and Yin]{Fan2019GraphNN}
W.~Fan, Y.~Ma, Q.~Li, Y.~He, Y.~E. Zhao, J.~Tang, and D.~Yin.
\newblock Graph neural networks for social recommendation.
\newblock \emph{Proceedings of the Web Conference}, 2019.

\bibitem[Fan et~al.(2021)Fan, Liu, Xie, Liu, Xiong, and Fu]{fan2021fair}
W.~Fan, K.~Liu, R.~Xie, H.~Liu, H.~Xiong, and Y.~Fu.
\newblock Fair graph auto-encoder for unbiased graph representations with {Wasserstein} distance.
\newblock In \emph{International Conference on Data Mining}, 2021.

\bibitem[Gkartzios et~al.(2025)Gkartzios, Pitoura, and Tsaparas]{gkartzios2025fair}
C.~Gkartzios, E.~Pitoura, and P.~Tsaparas.
\newblock Fair network communities through group modularity.
\newblock In \emph{Proceedings of the Web Conference}, 2025.

\bibitem[Guo et~al.(2023)Guo, Li, Xiao, Ma, and Wang]{23CAF}
Z.~Guo, J.~Li, T.~Xiao, Y.~Ma, and S.~Wang.
\newblock Towards fair graph neural networks via graph counterfactual.
\newblock In \emph{International Conference on Information and Knowledge Management}, pages 669--678, 2023.

\bibitem[Guo et~al.(2025)Guo, Wu, Xiao, Aggarwal, Liu, and Wang]{guo2023counterfactual}
Z.~Guo, Z.~Wu, T.~Xiao, C.~Aggarwal, H.~Liu, and S.~Wang.
\newblock Counterfactual learning on graphs: A survey.
\newblock \emph{Machine Intelligence Research}, 22\penalty0 (1):\penalty0 17--59, 2025.

\bibitem[Han et~al.(2023)Han, Zhang, Wu, and Yuan]{2023CFAD}
X.~Han, L.~Zhang, Y.~Wu, and S.~Yuan.
\newblock Achieving counterfactual fairness for anomaly detection.
\newblock In \emph{PAKDD}, 2023.

\bibitem[Hardt et~al.(2016)Hardt, Price, and Srebro]{Hardt2016EqualityOO}
M.~Hardt, E.~Price, and N.~Srebro.
\newblock Equality of opportunity in supervised learning.
\newblock In \emph{Neural Information Processing Systems}, 2016.

\bibitem[Hashimoto et~al.(2018)Hashimoto, Srivastava, Namkoong, and Liang]{Hashimoto2018FairnessWD}
T.~Hashimoto, M.~Srivastava, H.~Namkoong, and P.~Liang.
\newblock Fairness without demographics in repeated loss minimization.
\newblock In \emph{International Conference on Machine Learning}, 2018.

\bibitem[He et~al.(2024)He, Xu, Jiang, Wang, and Huang]{He2023ADAGADAA}
J.~He, Q.~Xu, Y.~Jiang, Z.~Wang, and Q.~Huang.
\newblock {ADA-GAD}: Anomaly-denoised autoencoders for graph anomaly detection.
\newblock In \emph{AAAI Conference on Artificial Intelligence}, 2024.

\bibitem[Huang et~al.(2023)Huang, Wang, Zhang, and Lin]{huang2023unsupervised}
Y.~Huang, L.~Wang, F.~Zhang, and X.~Lin.
\newblock Unsupervised graph outlier detection: Problem revisit, new insight, and superior method.
\newblock In \emph{International Conference on Data Engineering}, 2023.

\bibitem[Kipf and Welling(2017)]{Kipf:2017tc}
T.~N. Kipf and M.~Welling.
\newblock {Semi-Supervised Classification with Graph Convolutional Networks}.
\newblock In \emph{International Conference on Learning Representations}, 2017.

\bibitem[Kose and Shen(2022)]{kose2022fair}
O.~D. Kose and Y.~Shen.
\newblock Fair node representation learning via adaptive data augmentation.
\newblock \emph{arXiv:2201.08549}, 2022.

\bibitem[Kusner et~al.(2017)Kusner, Loftus, Russell, and Silva]{2017counterfactual}
M.~J. Kusner, J.~Loftus, C.~Russell, and R.~Silva.
\newblock Counterfactual fairness.
\newblock In \emph{Neural Information Processing Systems}, 2017.

\bibitem[Li et~al.(2017)Li, Dani, Hu, and Liu]{li2017radar}
J.~Li, H.~Dani, X.~Hu, and H.~Liu.
\newblock Radar: Residual analysis for anomaly detection in attributed networks.
\newblock In \emph{International Joint Conference on Artificial Intelligence}, 2017.

\bibitem[Liu et~al.(2022{\natexlab{a}})Liu, Dou, Zhao, Ding, Hu, Zhang, Ding, Chen, Peng, Shu, Sun, Li, Chen, Jia, and Yu]{22Pygod}
K.~Liu, Y.~Dou, Y.~Zhao, X.~Ding, X.~Hu, R.~Zhang, K.~Ding, C.~Chen, H.~Peng, K.~Shu, L.~Sun, J.~Li, G.~H. Chen, Z.~Jia, and P.~S. Yu.
\newblock {BOND}: Benchmarking unsupervised outlier node detection on static attributed graphs.
\newblock In \emph{Neural Information Processing Systems}, 2022{\natexlab{a}}.

\bibitem[Liu et~al.(2022{\natexlab{b}})Liu, Dou, Zhao, Ding, Hu, Zhang, Ding, Chen, Peng, Shu, Sun, Li, Chen, Jia, and Yu]{Liu2022BONDBU}
K.~Liu, Y.~Dou, Y.~Zhao, X.~Ding, X.~Hu, R.~Zhang, K.~Ding, C.~Chen, H.~Peng, K.~Shu, L.~Sun, J.~Li, G.~H. Chen, Z.~Jia, and P.~S. Yu.
\newblock {BOND}: Benchmarking unsupervised outlier node detection on static attributed graphs.
\newblock In \emph{Neural Information Processing Systems}, 2022{\natexlab{b}}.

\bibitem[Liu et~al.(2021)Liu, Li, Pan, Gong, Zhou, and Karypis]{22Cola}
Y.~Liu, Z.~Li, S.~Pan, C.~Gong, C.~Zhou, and G.~Karypis.
\newblock Anomaly detection on attributed networks via contrastive self-supervised learning.
\newblock \emph{IEEE transactions on neural networks and learning systems}, 33\penalty0 (6):\penalty0 2378--2392, 2021.

\bibitem[Liu et~al.(2023)Liu, Nguyen, and Fang]{liu2023generalized}
Z.~Liu, T.-K. Nguyen, and Y.~Fang.
\newblock On generalized degree fairness in graph neural networks.
\newblock In \emph{AAAI Conference on Artificial Intelligence}, 2023.

\bibitem[Ma et~al.(2022)Ma, Guo, Wan, Yang, Zhang, and Li]{22GEAR}
J.~Ma, R.~Guo, M.~Wan, L.~Yang, A.~Zhang, and J.~Li.
\newblock Learning fair node representations with graph counterfactual fairness.
\newblock In \emph{WSDM}, 2022.

\bibitem[Ma et~al.(2023)Ma, Wu, Xue, Yang, Zhou, Sheng, Xiong, and Akoglu]{gad}
X.~Ma, J.~Wu, S.~Xue, J.~Yang, C.~Zhou, Q.~Z. Sheng, H.~Xiong, and L.~Akoglu.
\newblock A comprehensive survey on graph anomaly detection with deep learning.
\newblock \emph{IEEE Transactions on Knowledge and Data Engineering}, 35\penalty0 (12):\penalty0 12012--12038, 2023.

\bibitem[Neo et~al.(2024)Neo, Lee, Jin, Kim, and Kumar]{Neo2024TowardsFG}
N.~K.~N. Neo, Y.-C. Lee, Y.~Jin, S.-W. Kim, and S.~Kumar.
\newblock Towards fair graph anomaly detection.
\newblock In \emph{International Conference on Information and Knowledge Management}, 2024.

\bibitem[Ng et~al.(2019)Ng, Zhu, Chen, and Fang]{ng2019graph}
I.~Ng, S.~Zhu, Z.~Chen, and Z.~Fang.
\newblock A graph autoencoder approach to causal structure learning.
\newblock In \emph{NeurIPS Workshop on Machine Learning and Causal Inference for Improved Decision Making}, 2019.

\bibitem[Pearl(2009)]{10.5555/1642718}
J.~Pearl.
\newblock \emph{Causality: Models, Reasoning and Inference}.
\newblock Cambridge University Press, USA, 2nd edition, 2009.

\bibitem[Rahman et~al.(2019)Rahman, Surma, Backes, and Zhang]{rahman2019fairwalk}
T.~Rahman, B.~Surma, M.~Backes, and Y.~Zhang.
\newblock Fairwalk: Towards fair graph embedding.
\newblock In \emph{International Joint Conference on Artificial Intelligence}, 2019.

\bibitem[Ren et~al.(2023)Ren, Xia, Lee, Noori~Hoshyar, and Aggarwal]{10.1145/3570906}
J.~Ren, F.~Xia, I.~Lee, A.~Noori~Hoshyar, and C.~Aggarwal.
\newblock Graph learning for anomaly analytics.
\newblock \emph{ACM Transactions on Intelligent Systems and Technology}, 14\penalty0 (2):\penalty0 1--29, 2023.

\bibitem[Ren et~al.(2024)Ren, Li, Peng, Xiang, Qin, and Ren]{2024mirror}
Z.~Ren, X.~Li, J.~Peng, Y.~Xiang, Z.~Qin, and K.~Ren.
\newblock Graph autoencoder with mirror temporal convolutional networks for traffic anomaly detection.
\newblock \emph{Scientific Reports}, 14\penalty0 (1):\penalty0 1247, 2024.

\bibitem[Roy et~al.(2024)Roy, Shu, Li, Yang, Elshocht, Smeets, and Li]{24GADNR}
A.~Roy, J.~Shu, J.~Li, C.~Yang, O.~Elshocht, J.~Smeets, and P.~Li.
\newblock Gad-nr: Graph anomaly detection via neighborhood reconstruction.
\newblock In \emph{WSDM}, 2024.

\bibitem[Sakurada and Yairi(2014)]{Sakurada2014AnomalyDU}
M.~Sakurada and T.~Yairi.
\newblock Anomaly detection using autoencoders with nonlinear dimensionality reduction.
\newblock In \emph{The Second Workshop on Machine Learning for Sensory Data Analysis}, 2014.

\bibitem[Traud et~al.(2012)Traud, Mucha, and Porter]{traud2012social}
A.~L. Traud, P.~J. Mucha, and M.~A. Porter.
\newblock Social structure of facebook networks.
\newblock \emph{Physica A: Statistical Mechanics and its Applications}, 391\penalty0 (16):\penalty0 4165--4180, 2012.

\bibitem[Valko and Hauskrecht(2008)]{valko2008distance}
M.~Valko and M.~Hauskrecht.
\newblock Distance metric learning for conditional anomaly detection.
\newblock In \emph{Proceedings of the International Florida AI Research Society Conference}, volume~21, pages 684--689, 2008.

\bibitem[Van~der Maaten and Hinton(2008)]{van2008visualizing}
L.~Van~der Maaten and G.~Hinton.
\newblock Visualizing data using {t-SNE}.
\newblock \emph{Journal of machine learning research}, 9\penalty0 (11), 2008.

\bibitem[Wang et~al.(2022)Wang, Zhang, Xiao, and Song]{Wang2021ARO}
J.~Wang, S.~Zhang, Y.~Xiao, and R.~Song.
\newblock A review on graph neural network methods in financial applications.
\newblock \emph{Journal of Data Science}, 20\penalty0 (2):\penalty0 111--134, 2022.

\bibitem[Wang et~al.(2024)Wang, Chu, Blanco, Chen, Chen, and Zhang]{wang2024advancing}
Z.~Wang, Z.~Chu, R.~Blanco, Z.~Chen, S.-C. Chen, and W.~Zhang.
\newblock Advancing graph counterfactual fairness through fair representation learning.
\newblock In \emph{ECML PKDD 2024}. 2024.

\bibitem[Wu et~al.(2021)Wu, Chen, Shao, Hong, Wang, and Wang]{wu2021learning}
L.~Wu, L.~Chen, P.~Shao, R.~Hong, X.~Wang, and M.~Wang.
\newblock Learning fair representations for recommendation: A graph-based perspective.
\newblock In \emph{Proceedings of the Web Conference}, 2021.

\bibitem[Xu et~al.(2022)Xu, Huang, Zhao, Dong, and Li]{22CONAD}
Z.~Xu, X.~Huang, Y.~Zhao, Y.~Dong, and J.~Li.
\newblock Contrastive attributed network anomaly detection with data augmentation.
\newblock In \emph{Pacific-Asia conference on knowledge discovery and data mining}, 2022.

\bibitem[Yang et~al.(2021)Yang, Liu, Chen, Shen, Hao, and Wang]{yang2020causalvae}
M.~Yang, F.~Liu, Z.~Chen, X.~Shen, J.~Hao, and J.~Wang.
\newblock Causal{VAE}: Disentangled representation learning via neural structural causal models.
\newblock In \emph{Proceedings of the IEEE/CVF conference on computer vision and pattern recognition}, pages 9593--9602, 2021.

\bibitem[Zemel et~al.(2013)Zemel, Wu, Swersky, Pitassi, and Dwork]{Zemel2013LearningFR}
R.~S. Zemel, L.~Y. Wu, K.~Swersky, T.~Pitassi, and C.~Dwork.
\newblock Learning fair representations.
\newblock In \emph{International Conference on Machine Learning}, 2013.

\bibitem[Zhang et~al.(2025)Zhang, Yuan, Cheng, Liu, Li, and Zhang]{Zhang2025Disentangled}
G.~Zhang, G.~Yuan, D.~Cheng, L.~Liu, J.~Li, and S.~Zhang.
\newblock Disentangled contrastive learning for fair graph representations.
\newblock \emph{Neural Networks}, 181:\penalty0 106781, 2025.

\bibitem[Zhu et~al.(2024)Zhu, Li, Zheng, and Chen]{zhu2024fair}
Y.~Zhu, J.~Li, Z.~Zheng, and L.~Chen.
\newblock Fair graph representation learning via sensitive attribute disentanglement.
\newblock In \emph{Proceedings of the Web Conference}, 2024.

\end{thebibliography}

\appendix

\newpage

\onecolumn

\section*{Appendix}

\section{Additional Experimental Results}
\subsection{Additional Injection Methods}\label{app:additional_exp}
In this section, we report the performance of DECAF-GAD on the same datasets, but with different outlier injection methods taken from \cite{22Pygod} than reported in the main text. We refer to the outlier injection method described in the main text as structural outlier. Another type of outlier is contextual outlier, where the features of selected nodes are replaced with features from farthest nodes.

\begin{table}[h]
\centering
\vskip 1em
\caption{Performance comparison on real-world datasets with contextual outliers. $\uparrow$ denotes that higher values indicate better performance, $\downarrow$ denotes that lower values indicate better performance. The better results between baseline and DECAF-GAD for each scenario are highlighted in \textbf{bold}. GADNR showed numerical instability for the Credit dataset and is not reported.}
\label{tab:main_results_real_ctxt}
\vskip 1em
\begin{tabular}{llccccc}
\toprule
Dataset & Method & Accuracy (\%) $\uparrow$ & F1-score (\%) $\uparrow$ & AUROC (\%) $\uparrow$ & $\Delta_{\text{EOO}}$ (\%) $\downarrow$ & $\Delta_{\text{DP}}$ (\%) $\downarrow$ \\
\midrule
\multirow{6}{*}{Bail} 
 & DOMINANT        & \textbf{93.70 $\pm$ 0.03} & \textbf{37.00 $\pm$ 0.33} & \textbf{66.84 $\pm$ 0.17} & 1.03 $\pm$ 0.61   & \textbf{0.75 $\pm$ 0.04} \\
 & DECAF-DOMINANT  & 93.66 $\pm$ 0.02 & 36.62 $\pm$ 0.23 & 66.64 $\pm$ 0.12 & \textbf{0.46 $\pm$ 0.30} & 0.82 $\pm$ 0.05 \\
 \cmidrule(l){2-7}
 & DONE            & \textbf{95.51 $\pm$ 0.29} & \textbf{55.06 $\pm$ 2.94} & \textbf{76.35 $\pm$ 1.55} & 3.87 $\pm$ 1.61 & \textbf{0.53 $\pm$ 0.11} \\
 & DECAF-DONE      & 93.91 $\pm$ 0.11 & 39.11 $\pm$ 1.06 & 67.95 $\pm$ 0.56 & \textbf{2.18 $\pm$ 1.34} & \textbf{0.53 $\pm$ 0.32} \\
 \cmidrule(l){2-7}
 & GADNR           & \textbf{93.42 $\pm$ 0.18} & \textbf{34.22 $\pm$ 1.83} & \textbf{65.38 $\pm$ 0.96} & 5.07 $\pm$ 3.88 & \textbf{0.37 $\pm$ 0.24} \\
 & DECAF-GADNR     & 92.50 $\pm$ 0.02 & 25.02 $\pm$ 0.19 & 60.54 $\pm$ 0.10 & \textbf{1.06 $\pm$ 0.23} & 0.66 $\pm$ 0.02 \\
\midrule
\multirow{6}{*}{German} 
 & DOMINANT        & 95.20 $\pm$ 0.00 & 52.00 $\pm$ 0.00 & 74.74 $\pm$ 0.00 & 2.94 $\pm$ 0.00 & \textbf{1.19 $\pm$ 0.00} \\
 & DECAF-DOMINANT  & \textbf{95.76 $\pm$ 0.08} & \textbf{57.60 $\pm$ 0.80} & \textbf{77.68 $\pm$ 0.42} & \textbf{2.13 $\pm$ 0.88} & \textbf{1.19 $\pm$ 0.00} \\
 \cmidrule(l){2-7}
 & DONE            & \textbf{96.28 $\pm$ 0.16} & \textbf{62.80 $\pm$ 1.60} & \textbf{80.42 $\pm$ 0.84} & 2.21 $\pm$ 0.93 & 0.35 $\pm$ 0.19 \\
 & DECAF-DONE      & 96.20 $\pm$ 0.00 & 62.00 $\pm$ 0.00 & 80.00 $\pm$ 0.00 & \textbf{0.74 $\pm$ 0.00} & \textbf{0.30 $\pm$ 0.14} \\
 \cmidrule(l){2-7}
 & GADNR           & 91.18 $\pm$ 1.03 & 11.80 $\pm$ 10.33 & 53.58 $\pm$ 5.44 & 2.54 $\pm$ 3.60 &\textbf{ 0.75 $\pm$ 0.45} \\
 & DECAF-GADNR     & \textbf{92.40 $\pm$ 0.00} & \textbf{24.00 $\pm$ 0.00} & \textbf{60.00 $\pm$ 0.00} & \textbf{1.47 $\pm$ 0.00} & 1.19 $\pm$ 0.00 \\
\midrule
\multirow{6}{*}{Credit} 
 & DOMINANT        & \textbf{94.53 $\pm$ 0.09} & \textbf{45.33 $\pm$ 0.87} & \textbf{71.23 $\pm$ 0.46} & 11.37 $\pm$ 1.05 & 6.53 $\pm$ 0.53 \\
 & DECAF-DOMINANT  & 91.52 $\pm$ 0.43 & 15.14 $\pm$ 4.34 & 55.34 $\pm$ 2.29 & \textbf{7.48 $\pm$ 5.02} & \textbf{4.67 $\pm$ 2.87} \\
 \cmidrule(l){2-7}
 & DONE            & \textbf{99.89 $\pm$ 0.01} & \textbf{98.89 $\pm$ 0.15} & \textbf{99.41 $\pm$ 0.08} & \textbf{1.25
 $\pm$ 1.03 }& 4.79 $\pm$ 0.12 \\
 & DECAF-DONE      & 90.99  $\pm$ 0.1 & 9.88 $\pm$ 1.0 & 52.57 $\pm$ 0.53 & 3.84 $\pm$ 1.49 & \textbf{4.13   $\pm$ 0.21} \\
\bottomrule
\end{tabular}
\end{table}

\begin{table}[h]
\centering
\caption{Performance comparison on the synthetic dataset with contextual outliers. $\uparrow$ denotes that higher values indicate better performance, $\downarrow$ denotes that lower values indicate better performance. The better results between baseline and DECAF-GAD for each scenario are highlighted in \textbf{bold}.}
\label{tab:main_results_synt_ctxt}
\begin{tabular}{lcccccc}
\hline
Method & Accuracy (\%) $\uparrow$ & F1-score (\%) $\uparrow$ & AUROC (\%) $\uparrow$ & $\Delta_{\text{EOO}}$ (\%) $\downarrow$ & $\Delta_{\text{DP}}$ (\%) $\downarrow$ & $\Delta_{\text{CF}}$ (\%) $\downarrow$ \\
\hline
DOMINANT & \textbf{96.00 $\pm$ 0.25} & \textbf{60.00 $\pm$ 2.53} & \textbf{78.95 $\pm$ 1.33} & 8.11 $\pm$ 3.76 & 3.70 $\pm$ 0.83 & \textbf{4.17 $\pm$ 0.33} \\
DECAF-DOMINANT & 93.92 $\pm$ 0.32 & 39.22 $\pm$ 3.22 & 68.01 $\pm$ 1.69 & \textbf{8.02 $\pm$ 3.29} & \textbf{2.45 $\pm$ 0.45} & 5.34 $\pm$ 0.24 \\
\hline  
DONE & 93.93 $\pm$ 0.44 & 39.30 $\pm$ 4.38 & 68.05 $\pm$ 2.31 & 8.45 $\pm$ 6.24 & 1.96 $\pm$ 0.57 & 6.50 $\pm$ 0.25 \\
DECAF-DONE & \textbf{95.73 $\pm$ 0.91} & \textbf{57.30 $\pm$ 9.13} & \textbf{77.53 $\pm$ 4.81} & \textbf{7.45 $\pm$ 6.79} & \textbf{0.61 $\pm$ 0.44} & \textbf{4.11 $\pm$ 0.89} \\
\hline

GADNR & 90.80 $\pm$ 0.38 & 8.00 $\pm$ 3.85 & 51.58 $\pm$ 2.02 & 3.23 $\pm$ 2.93 & 4.05 $\pm$ 1.71 & 7.58 $\pm$ 0.86 \\
DECAF-GADNR & \textbf{95.26 $\pm$ 0.17} & \textbf{52.60 $\pm$ 1.74} & \textbf{75.05 $\pm$ 0.92} & \textbf{0.23 $\pm$ 0.12} & \textbf{27.05 $\pm$ 1.11} & \textbf{0.46 $\pm$ 0.19} \\
\hline

\end{tabular}
\end{table}

\subsection{Additional Sensitivity Analysis}\label{app:sensitivity}

\begin{figure}[h] 
    \centering
    \begin{subfigure}{0.33\linewidth} 
        \centering
        \includegraphics[width=\linewidth]{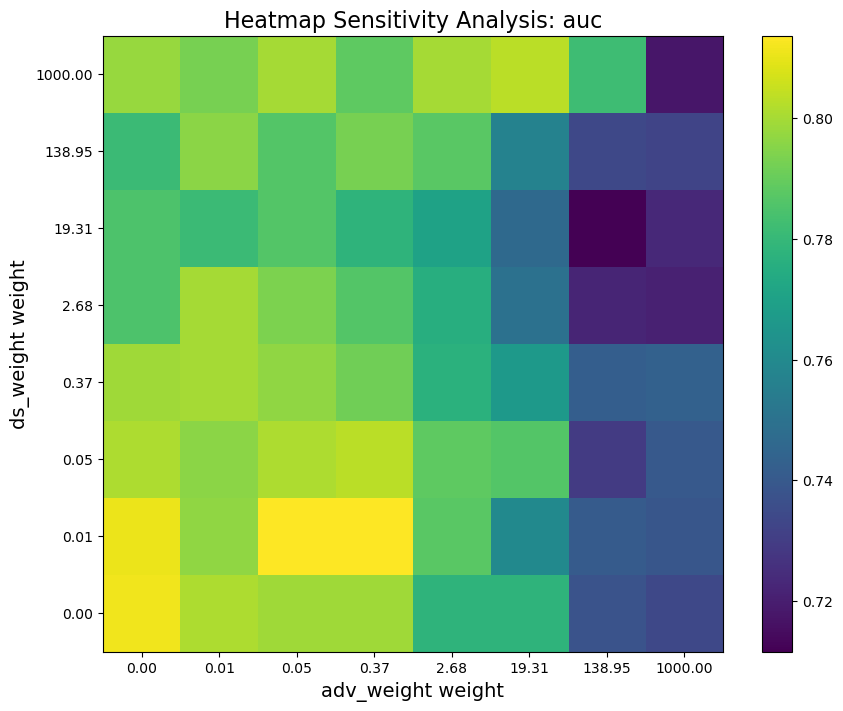} 
        \caption{$\lambda_1$ vs. $\lambda_2$}
    \end{subfigure}
    \hfill
    \begin{subfigure}{0.33\linewidth} 
        \centering
        \includegraphics[width=\linewidth]{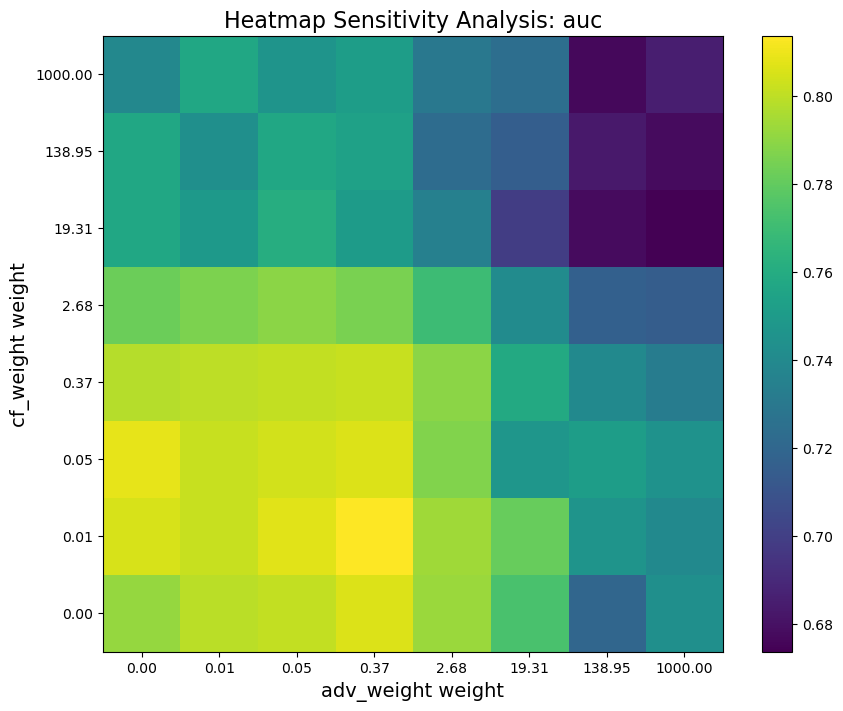} 
        \caption{$\lambda_2$ vs. $\lambda_3$}
    \end{subfigure}
    \hfill
    \begin{subfigure}{0.33\linewidth} 
        \centering
        \includegraphics[width=\linewidth]{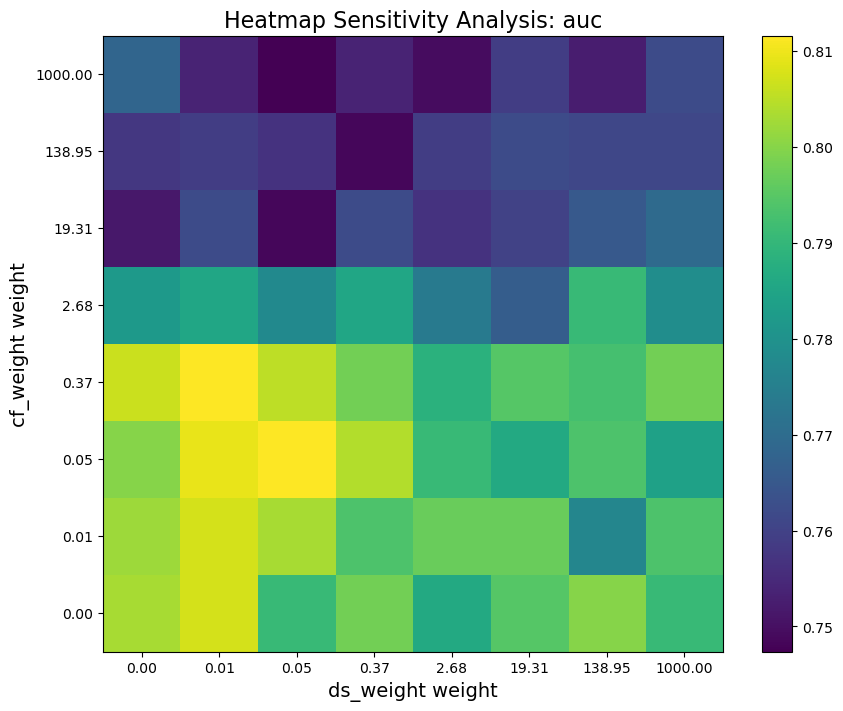} 
        \caption{$\lambda_1$ vs. $\lambda_3$}
    \end{subfigure}
    \vskip 2em
    \caption{Sensitivity analysis for AUC scores.}
    \label{fig:sensitivity_auc}
\end{figure}

We conduct a sensitivity analysis on the hyperparameters $\lambda_1$, $\lambda_2$ and $\lambda_3$ by examining their impact on AUC scores for the synthetic dataset using the DECAF-DOMINANT method. The results are visualized in Figure \ref{fig:sensitivity_auc}. Similar to $\Delta_\text{CF}$ scores, the AUC scores are also stable under the change of the hyperparameters.

\subsection{Loss Contribution Analysis}

Table \ref{tab:loss_contribution} reports the performance change when removing each fairness component from DECAF-DOMINANT on the Bail dataset. Removing either the counterfactual (CF) loss or the adversarial (ADV) loss generally reduces fairness (higher DP/EOO) and in some cases slightly impacts detection accuracy. The baseline (without fairness modules) achieves higher accuracy and AUC but substantially worse fairness.

\begin{table}[h]
\centering
\caption{Ablation study on Bail/DOMINANT showing the contribution of each loss to accuracy and fairness.}
\label{tab:loss_contribution}
\begin{tabular}{lccccc}
\toprule
\textbf{Setting} & \textbf{AUC} & \textbf{ACC} & \textbf{F1} & \textbf{DP} & \textbf{EOO} \\
\midrule
Full model      & 54.65 & 91.38 & 13.85 & 0.66 & 2.35 \\
No CF Loss      & 53.64 & 91.19 & 11.93 & 0.71 & 2.52 \\
No ADV Loss     & 53.86 & 91.23 & 12.35 & 0.83 & 2.59 \\
No CF \& ADV    & 54.60 & 91.37 & 13.74 & 0.78 & 3.34 \\
Baseline        & 57.38 & 91.90 & 19.04 & 0.84 & 2.64 \\
\bottomrule
\end{tabular}
\end{table}

\subsection{Impact of Split Ratio}

Table \ref{tab:split_ratio} examines the effect of different W1 split ratios for DECAF-DOMINANT on Bail. Across 2:8, 4:6, and 5:5 splits, accuracy and AUC remain similar, while fairness metrics (DP, EOO) vary moderately, with the 5:5 split yielding the best fairness balance.

\begin{table}[h]
\centering
\caption{Effect of split ratio on Bail/DOMINANT performance.}
\label{tab:split_ratio}
\begin{tabular}{lcccc}
\toprule
\textbf{Split Ratio} & \textbf{ACC} & \textbf{AUC} & \textbf{EOO} & \textbf{DP} \\
\midrule
2:8 & 91.57 & 55.68 & 3.92 & 0.56 \\
4:6 & 91.47 & 55.18 & 3.14 & 0.68 \\
5:5 & 91.49 & 55.23 & 2.41 & 0.61 \\
\bottomrule
\end{tabular}
\end{table}

\end{document}